\documentclass[12pt]{article}
\usepackage[margin=1in]{geometry}
\usepackage{setspace}

\usepackage{preamble}
\usepackage{comment}

\providecommand{\coloneqq}{\mathrel{\mathop:}=}

\DeclareRobustCommand{\UCBBQRL}{UCB\textendash BQRL}
\DeclareRobustCommand{\EVIBQ}{EVI\textendash BQ}

\newcommand{\Prob}{\mathbb P}
\newcommand{\Reg}{\operatorname{Reg}}
\newcommand{\eps}{\varepsilon}

\definecolor{algNavy}{HTML}{004488}
\definecolor{algOrange}{HTML}{E69F00}
\definecolor{algTeal}{HTML}{009E73}
\definecolor{algRed}{HTML}{D55E00}
\definecolor{algPurple}{HTML}{AA4499}
\definecolor{algBrown}{HTML}{8C564B}
\definecolor{algCyan}{HTML}{00A6D6}

\tikzset{
    policy legend swatch/.style={
        draw=black!55,
        line width=0.3pt,
        minimum width=0.38cm,
        minimum height=0.16cm,
        inner sep=0pt
    }
}

\pgfplotsset{
    professional legend/.style={
        draw=black!55,
        fill=white,
        fill opacity=0.97,
        text opacity=1,
        rounded corners=1pt,
        font=\scriptsize,
        inner xsep=3pt,
        inner ysep=2pt,
        legend cell align=left,
        row sep=0pt
    }
}

\newcommand{\QMDPDataDir}{.}
\newcommand{\cimult}{1.96}

\newcommand{\addregretseries}[6]{%

    \addplot[
        name path=#1-upper,
        draw=none,
        forget plot
    ]
    table[
        x=episode,
        y expr={
            \thisrow{#2_mean}
            + \cimult*\thisrow{#2_std}/sqrt(3)
        },
        col sep=comma
    ]{#6};

    \addplot[
        name path=#1-lower,
        draw=none,
        forget plot
    ]
    table[
        x=episode,
        y expr={
            \thisrow{#2_mean}
            - \cimult*\thisrow{#2_std}/sqrt(3)
        },
        col sep=comma
    ]{#6};

    \addplot[
        draw=none,
        fill=#3,
        fill opacity=0.06,
        forget plot
    ]
    fill between[
        of=#1-upper and #1-lower
    ];

    \addplot[
        color=#3,
        solid,
        mark=none,
        #4
    ]
    table[
        x=episode,
        y=#2_mean,
        col sep=comma
    ]{#6};

    \addlegendentry{#5}
}

\newcommand{\assetregretplot}[2][%
Cumulative numerical $\tau$-quantile policy gap ($\widehat{\mathcal G}_{\tau}(k)$)]{%

\begin{tikzpicture}

\begin{axis}[
    width=0.97\textwidth,
    height=0.52\textwidth,
    xmin=0,
    xmax=2000,
    xtick={0,250,...,2000},
    xlabel={Episode ($k$)},
    ylabel={#1},
    tick align=outside,
    axis line style={black!75},
    grid=major,
    grid style={
        draw=black!15,
        line width=0.25pt
    },
    tick label style={font=\small},
    label style={font=\small},
    scaled ticks=false,
    legend columns=4,
    legend image post style={xscale=0.70},
    legend style={
        at={(0.5,1.02)},
        anchor=south,
        draw=black!60,
        fill=white,
        rounded corners=1pt,
        font=\scriptsize,
        inner xsep=3pt,
        inner ysep=2pt,
        legend cell align=left,
        row sep=1pt,
        /tikz/every even column/.append style={column sep=2pt},
        /tikz/every odd column/.append style={column sep=7pt}
        }
]

\addregretseries
    {bq}
    {ucb_bqrl}
    {algNavy}
    {very thick}
    {\UCBBQRL}
    {#2}

\addregretseries
    {vi}
    {ucbvi}
    {algOrange}
    {thick}
    {\textsc{UCBVI}}
    {#2}

\addregretseries
    {eg}
    {eps_q}
    {algTeal}
    {thick}
    {$\varepsilon$-greedy Q-learning}
    {#2}

\addregretseries
    {sa}
    {sarsa}
    {algRed}
    {thick}
    {\textsc{SARSA}}
    {#2}

\addregretseries
    {th}
    {thompson}
    {algPurple}
    {thick}
    {Thompson}
    {#2}

\addregretseries
    {pp}
    {ppo}
    {algBrown}
    {thick}
    {\textsc{PPO}}
    {#2}

\addregretseries
    {tr}
    {trpo}
    {algCyan}
    {thick}
    {\textsc{TRPO}}
    {#2}

\end{axis}

\end{tikzpicture}%
}

\pgfplotstableread[col sep=comma]
{\QMDPDataDir/AssetSelling_tau0p1_UCB_BQRL_final_policy.csv}
\dataTauLow

\pgfplotstableread[col sep=comma]
{\QMDPDataDir/AssetSelling_tau0p5_UCB_BQRL_final_policy.csv}
\dataTauMedian

\pgfplotstableread[col sep=comma]
{\QMDPDataDir/AssetSelling_tau0p9_UCB_BQRL_final_policy.csv}
\dataTauHigh

\newcommand{\policyplot}[3][]{%

    \nextgroupplot[
        title={$#3$},
        #1
    ]

    \foreach \s in {0,...,24}{%
        \foreach \h in {1,...,10}{%

            \pgfplotstablegetelem{\s}{[index]\h}\of{#2}%
            \edef\cellvalue{\pgfplotsretval}%

            \pgfmathparse{\h-0.5}%
            \let\xleft\pgfmathresult

            \pgfmathparse{\h+0.5}%
            \let\xright\pgfmathresult

            \pgfmathparse{\s-0.5}%
            \let\ybottom\pgfmathresult

            \pgfmathparse{\s+0.5}%
            \let\ytop\pgfmathresult

            \ifnum\cellvalue=1\relax
                \def\cellcolor{algNavy}%
            \else
                \def\cellcolor{black!18}%
            \fi

            \edef\drawcell{%
                \noexpand\path[
                    fill=\cellcolor,
                    draw=white,
                    line width=0.12pt
                ]
                (axis cs:\xleft,\ybottom)
                rectangle
                (axis cs:\xright,\ytop);%
            }%

            \drawcell
        }%
    }%
}

\makeatletter

\newenvironment{breakablealgorithm}
  {%
   \begin{center}
   \refstepcounter{algorithm}
   \hrule height.8pt depth0pt
   \kern2pt

   \renewcommand{\caption}[2][\relax]{%
     {\raggedright
      \textbf{\ALG@name~\thealgorithm} ##2\par}%

     \ifx\relax##1\relax
       \addcontentsline{loa}{algorithm}{%
         \protect\numberline{\thealgorithm}##2}%
     \else
       \addcontentsline{loa}{algorithm}{%
         \protect\numberline{\thealgorithm}##1}%
     \fi

     \kern2pt
     \hrule
     \kern2pt
   }}
  {%
   \kern2pt
   \hrule
   \relax
   \end{center}
  }

\makeatother
\usepackage{microtype}
\newcommand{\myqed}{\unskip\nobreak\hfill\ensuremath{\square}}

\begin{document}

\title{\fontsize{18pt}{20pt}\selectfont \bf
   Risk-Sensitive Reinforcement Learning with Smoothed Quantile Objectives
}

\author{\fontsize{12pt}{16pt}
    Mohammad Alipour-Vaezi, Huaiyang Zhong, Sajad Khodadadian\thanks{Corresponding Author, 
    \\Email Addresses: alipourvaezi@vt.edu (M. Alipour-Vaezi); hzhong@vt.edu (H. Zhong); sajadk@vt.edu (S. Khodadadian)\\
    ORCID IDs: 0000-0002-7529-1848 (M. Alipour-Vaezi); 0000-0002-2902-1644 (H. Zhong); 0000-0002-5197-4652 (S. Khodadadian)
.}
}

\date{
    \textsuperscript{}\textit{Grado Department of Industrial \& Systems Engineering, Virginia Tech, Blacksburg, VA 24061, USA}
}

\maketitle
\vspace{-5mm}
\begin{abstract}
\noindent
Reinforcement Learning (RL) has achieved tremendous success in recent years. However, the classical foundations of RL do not account for the risk sensitivity of the objective function, which is critical in various fields, including healthcare, finance, etc. A popular approach to incorporate risk sensitivity is to optimize a specific quantile of the cumulative reward distribution. However, exact quantile objectives are non-smooth and can change abruptly under small perturbations of the return distribution, making them difficult to optimize reliably when the transition model must be learned from data. Motivated by this instability, we develop \UCBBQRL, a model-based optimistic learning algorithm that maintains confidence sets for the transition kernel and plans using a lower-buffered quantile criterion. The buffered criterion smooths the exact quantile objective by averaging nearby lower quantiles, thereby improving stability under transition-estimation error. To compute the buffered-quantile policy at each episode, we introduce \EVIBQ, an exact dynamic-programming procedure. We establish a high-probability regret bound for \UCBBQRL, which up to logarithmic factors scales as $\mathcal{O}(\mathrm{e}^{\tau/\rho_\tau}+H^2\sqrt{SAT})$, where $\rho_\tau$ is denoted as the root-level left-plateau threshold, which is a problem-dependent constant. Further, we establish an information-theoretic lower bound of $\Omega(H/\rho_\tau\sqrt{AT})$ for the regret of any algorithm dealing with a quantile objective function. Finally, we prove that the exact point-quantile
evaluation and exact lower-buffered quantile evaluation are PP-hard under polynomial-time Turing reductions, even for a fixed policy in a two-state, one-action finite-horizon MDP.
\end{abstract}


\doublespacing
\newpage
\section{Introduction}

\par Reinforcement learning (RL) provides a general framework for sequential decision
making, where a learner interacts with an unknown environment and
improves its policy from observed trajectories \cite{sutton1998reinforcement}.
The dominant theoretical and algorithmic foundation of RL is built around the
maximization of expected cumulative reward. This expectation-based criterion has led to algorithms with strong regret guarantees \cite{jaksch2010near,azar2017minimax} and broad empirical success in applications such as video games, robotic manipulation, news recommendation, and inventory management \cite{mnih2015human,levine2016end,shen2017deep,gijsbrechts2022can}.
However, expectation alone can be an inadequate performance criterion in settings
where the reliability, downside behavior, or tail performance of the return
distribution is central to decision quality. In safety-critical control, healthcare, finance, and service systems, decision makers often care not only about expected performance but also about satisfying safety constraints or achieving specified tail-performance and service-level guarantees \cite{garcia2015comprehensive,gottesman2019guidelines,li2022quantile}.

\par A natural way to encode such distributional preferences is through a quantile
objective. Rather than maximizing the expected return, the learner seeks a policy
that maximizes a fixed $\tau$-quantile of the cumulative reward distribution, where the quantile operator is defined as
$Q_\tau(X)\coloneqq \inf\{x\in\mathbb R:\ \Prob(X\le x)\ge \tau\}$.
This criterion directly captures tail-sensitive behavior: small values of
$\tau$ emphasize conservative or downside performance, while larger values of
$\tau$ emphasize more aggressive upper-tail outcomes. Quantile objectives are
closely related to Value-at-Risk (VaR) and have long been used in finance,
service-level analysis, and risk-sensitive decision making \cite{li2022quantile}. 
Under a known transition kernel, \cite{li2022quantile} introduced the Quantile Markov Decision Process (QMDP), which provides a backward dynamic
programming framework for optimizing the \(\tau\)-quantile value of the cumulative reward.
 A key feature of QMDP is that the optimal policy is
generally not Markovian in the state alone; instead, it depends on evolving quantile levels that must be tracked as part of the state.
This augmentation makes it possible to solve a fundamentally non-Markovian
objective through a backward recursion over state--quantile pairs.

\par Despite this progress in model-based planning with known dynamics, the online learning problem for quantile objectives remains substantially less understood. In the classical expectation-based setting, a successful approach to exploration is the principle of optimism in the face of uncertainty, which underlies algorithms such as UCRL2~\cite{jaksch2010near} and UCBVI~\cite{azar2017minimax}. These methods maintain a set of statistically plausible transition models and select a policy that performs optimistically for some model in this set, thereby encouraging exploration of uncertain parts of the system while exploiting the information gathered so far. Extending this idea to quantile objectives is challenging because quantiles are nonlinear and can be highly sensitive to small changes in the return distribution. A small transition-estimation error may move probability mass across the target quantile level and substantially change the value of a policy. Therefore, online learning with quantile objectives requires new algorithmic and analytical tools beyond those used for expectation-based reinforcement learning.

\par  This paper develops a learning framework for finite-horizon MDPs with a fixed quantile objective and unknown transition probabilities. We consider an episodic setting in which the learner repeatedly interacts with the environment, observes sample trajectories, updates its estimate of the transition law, and then selects a new policy for the next episode. The proposed algorithm, \UCBBQRL, follows a simple high-level principle: among the transition models that remain consistent with the observed data, choose a policy whose risk-sensitive performance is most favorable. To make this method stable for quantile objectives, the algorithm plans with a lower-buffered quantile function, which averages ordinary quantiles over a short interval below the target level. This buffered function smooths the quantile criterion while preserving a direct connection to the exact $\tau$-quantile objective used to define regret.

\par Our main results provide both a finite-sample high-probability upper bound and an information-theoretic lower bound for online learning with quantile objectives. For \UCBBQRL, the high-probability regret upper bound consists of a cumulative buffering term and a statistical learning term whose dominant order is  $\mathcal{O}(\mathrm{e}^{\tau/\rho_\tau}+H^2\sqrt{SAT})$,  where $S$ is the number of states, $A$ is the number of actions, $H$ is the horizon length, $T$ is the number of episodes, and $\rho_\tau$ is denoted as the \textit{root-level left-plateau threshold}, which is a problem-dependent constant. This yields sublinear exact quantile regret under a logarithmically decreasing buffer schedule, up to a finite instance-dependent transient term. Complementing this result, we prove that for every learning algorithm there exist finite-horizon MDP instances for which the expected exact quantile regret is at least $\Omega(H/\rho_\tau\sqrt{AT})$. This lower bound shows that the dependence on $1/\rho_\tau$ is unavoidable and identifies a fundamental statistical barrier. We further show that exact point-quantile evaluation and exact lower-buffered
quantile evaluation are PP-hard under polynomial-time Turing reductions, even for
a fixed policy in a two-state, one-action finite-horizon MDP. This separates the
statistical role of the exact planning oracle from its computational
tractability in general instances.

\paragraph{Contributions.} \begin{itemize} \item \textbf{Online learning for quantile-sensitive control.} We formulate an episodic finite-horizon reinforcement learning problem in which the objective is to maximize a fixed $\tau$-quantile of the cumulative reward distribution under unknown transition dynamics. \item \textbf{A buffered optimistic learning algorithm.} We introduce \UCBBQRL, a model-based learning algorithm that combines transition confidence sets with a lower-buffered smooth approximation of quantile objective. 
\item \textbf{An exact planning procedure for the buffered objective.} We provide \EVIBQ, a dynamic-programming procedure for solving the buffered quantile planning problem that arises inside \UCBBQRL. 
\item \textbf{A high-probability regret bound.} We establish a finite-sample high-probability regret bound for \UCBBQRL. 
\item \textbf{An information-theoretic lower bound.} We prove a lower bound showing that, for some finite-horizon MDPs, any learning algorithm must incur expected exact quantile regret of order \(\Omega\!\left(\frac{H}{\rho_\tau}\sqrt{AT}\right)\), where $\rho_\tau$ measures the relevant root-level quantile plateau. 
\item \textbf{Computational hardness of exact quantile evaluation.} We prove that
exact point-quantile evaluation and exact lower-buffered quantile evaluation are
PP-hard under polynomial-time Turing reductions, even for a fixed policy in a
two-state, one-action finite-horizon MDP.
\end{itemize}

\section{Related Work}
\subsection{Risk-Sensitive MDPs}\label{subsec:rsmdp} 
Risk-sensitive objectives in sequential decision-making have been examined across several methodological paradigms. Foundational studies on percentile and quantile criteria investigated the existence, structural properties, and computational aspects of such objectives in controlled Markov processes with known dynamics, including shortest-path and service-level formulations \cite{filar1995percentile,delage2010percentile}. The Quantile MDP (QMDP) framework further formalizes dynamic programming for fixed quantile levels and establishes backward-recursion procedures and planning algorithms under known transition kernels \cite{li2022quantile}. More recently, \cite{alipour-vaezi2026optimistic} incorporated the quantile risk measure into the RL setting under unknown kernels. A closely related line of work, distributional reinforcement learning, propagates the full return distribution and has led to practical quantile-based parameterizations, including QR-DQN and IQN \cite{bellemare2017distributional,dabney2018distributional,dabney2018implicit,rowland2018analysis,yang2019fully}. Although distributional reinforcement learning methods typically optimize the \emph{expected} return, their estimators offer useful tools for learning quantile slices.

Beyond quantile-based objectives, classical risk-sensitive control has traditionally focused on exponential-utility, or entropic, criteria, which give rise to modified Bellman equations and preserve dynamic consistency \cite{howard1972risk}. Mean--variance MDPs examine trade-offs between expected return and return variance, although they generally suffer from time inconsistency in the absence of special structural conditions \cite{sobel1982variance,mannor2011mean,guo2012mean}. Coherent risk measures, particularly Conditional Value-at-Risk (CVaR) \cite{rockafellar2000optimization,rockafellar2002conditional}, provide convex surrogate formulations and have been extensively studied in reinforcement learning through value-based, policy-gradient, and actor--critic approaches, both as optimization objectives and as constraints \cite{chow2014algorithms,tamar2015optimizing,prashanth2014policy}. Constrained MDPs (CMDPs) and safe reinforcement learning further incorporate chance-type or CVaR-type constraints through Lagrangian, primal--dual, and Lyapunov-based methods \cite{altman2021constrained,Chow2015Risk-Constrained,Zhang2024CVaR-Constrained,M2022Approximate,Ahmadi2020Constrained}. These research directions are largely complementary to our framework, which \emph{maximizes a fixed quantile objective} rather than imposing it as a constraint, thereby requiring direct treatment of the non-smooth and set-valued nature of the quantile backup operator. From a methodological perspective, quantile regression \cite{koenker1978regression} provides the foundation for many practical estimators employed in distributional and quantile reinforcement learning; however, most of this literature does not study online regret under unknown transition dynamics \cite{dabney2018distributional,dabney2018implicit,yang2019fully}.

\subsection{Optimism and Upper Confidence Bounds (UCB)}\label{subsec:ucb}
Optimism in the face of uncertainty has been central to deriving near-minimax regret guarantees in expectation-maximizing reinforcement learning, primarily through planning over confidence sets constructed around empirical transition models. In average-reward communicating MDPs, the UCRL2 algorithm obtains $\tilde{\mathcal O}(D S \sqrt{A T})$-type regret guarantees by employing $\ell_1$ confidence sets together with Extended Value Iteration (EVI) \cite{jaksch2010near}. For finite-horizon settings, the UCBVI algorithm achieves regret bounds of order $\tilde{\mathcal O}(H\sqrt{S A T})$ using Hoeffding-style bonuses and $\tilde{\mathcal O}(\sqrt{H S A T})$ using Bernstein-style bonuses \cite{azar2017minimax}. Robust and distributionally robust MDP frameworks similarly conduct planning against uncertainty sets at decision time, leading to max--min or ambiguity-aware backup operators that are algorithmically related to optimistic EVI-type subroutines \cite{Yu2015Distributionally,Goyal2022Robust,Xu2016Quantile,deo2025design}.

Adapting optimism to \emph{nonlinear} and tail-oriented criteria introduces additional technical challenges. Quantile objectives are inherently non-smooth and may vary discontinuously in response to small perturbations of the return distribution; consequently, the linear value-difference decompositions commonly used in expectation-based analyses do not directly extend to this setting. In one-step decision problems, the bandit literature has developed risk-aware index policies for VaR, CVaR, and more general risk measures \cite{sani2012risk,galichet2013exploration,cassel2023general}, thereby illustrating that confidence-set design must explicitly account for tail sensitivity. Extending these principles to MDPs requires new contraction and sensitivity arguments for the corresponding backup operator.

\subsection{Function Approximation and Model-Based Deep RL}
\label{subsec:function-approx-related}

Function approximation is central to reinforcement learning in large or continuous
state spaces, where tabular representations of value functions, policies, or
transition models are infeasible \cite{guo2021survey,
santamaria1997experiments,
jin2023provably}. Classical approximate dynamic programming and
fitted value-iteration methods approximate value functions using linear models,
basis functions, kernels, or supervised learning techniques
\cite{bertsekas2025neuro,
powell2007approximate,
munos2008finite}. Modern deep
RL extends this idea by using neural networks to parameterize value functions,
policies, and distributional return models
\cite{mnih2015human,
schulman2017proximal}.

Model-based deep RL instead learns an approximation of the transition dynamics
and uses the learned model for planning or policy improvement. Representative
methods combine neural dynamics models with trajectory optimization, model
predictive control, or synthetic rollouts to improve sample efficiency
\cite{deisenroth2011pilco,
chua2018deep,
kurutach2018model,
hafner2019learning}.
Because learned models may be inaccurate outside the observed data distribution,
ensembles, bootstrap methods, and Bayesian approximations are commonly used to
quantify epistemic uncertainty and support exploration
\cite{osband2016deep,lakshminarayanan2017simple,chua2018deep}. However,
most model-based deep RL methods are designed for expected-return optimization,
and their uncertainty mechanisms are typically heuristic or approximate
surrogates for confidence-set planning. Regret analysis for nonlinear,
tail-sensitive objectives under function approximation remains substantially
less developed.

\section{Preliminaries}\label{sec:prelims}
We study a finite-horizon MDP, denoted by
\(\mathcal{M} = (\mathcal{S}, \mathcal{A}, H, P^\star, r)\), \(r=\{r_h\}_{h=0}^{H-1}\), where
\(\mathcal{S}\) represents the state space, \(\mathcal{A}\) denotes the action space,
\(H\) is the planning horizon, \(P^\star_h(\cdot \mid s,a)\) is the true transition
kernel at step \(h\), and \(r_h(s,a)\in[0,1]\) denotes the deterministic reward at step \(h\) for every
state-action pair \((s,a)\). We assume that the state-action space is finite, and we
write \(S = |\mathcal{S}|\) and \(A = |\mathcal{A}|\).

\par For \(h=0,\ldots,H-1\), let $\mathcal H_h
\coloneqq \mathcal S\times(\mathcal A\times\mathcal S)^h$ denote the set of length-\(h\) physical histories. A deterministic
history-dependent policy is a sequence $\pi=\{\pi_h\}_{h=0}^{H-1},
\pi_h:\mathcal H_h\to\mathcal A$. We denote the class of deterministic history-dependent policies by \(\Pi_{\mathrm{det}}\). Since \(\mathcal S\), \(\mathcal A\), and \(H\) are
finite, the class \(\Pi_{\mathrm{det}}\) is finite.

\par In ordinary point-quantile QMDP, a deterministic history-dependent policy
may be implemented through a scalar state--quantile representation
\(\pi=\pi[\mu,\Gamma]\), where $\mu_h:\mathcal S\times[0,1]\to\mathcal A$, and 
$\Gamma_h:\mathcal S\times[0,1]\times\mathcal A\times\mathcal S\to[0,1]$. In that representation, \(q_h\) is a scalar quantile level. This scalar
representation is useful for ordinary QMDP, but it is not the representation
used by the exact lower-buffered frontier algorithms below.

\begin{definition}[Lower-buffered quantile]
\label{def:buffered-quantile}
For \(\beta\in(0,1)\) and \(q\in(0,1]\), define \(\ell_\beta(q)\coloneqq \min\{\beta,q\}.\) For a real-valued random variable \(X\), its lower-buffered \(q\)-quantile is $Q_q^\beta(X)\coloneqq\frac{1}{\ell_\beta(q)}
\int_{q-\ell_\beta(q)}^{q} Q_u(X)\,du$.
For the endpoint cases, we use the conventions \( Q_0^\beta(X) = \inf\operatorname{supp}(X)\) and \( Q_0^\beta(X) =\sup\operatorname{supp}(X)\).
\end{definition}
\par Quantile backups are non-smooth and may change discontinuously under small distributional perturbations. To obtain a tractable stability estimate, we stabilize the objective by using a lower-buffered quantile function as described in Definition~\ref{def:buffered-quantile}. This operator averages ordinary quantiles over a short interval immediately below the target quantile level, a form of local smoothing related to approaches studied in the risk-measure literature \cite{fissler2021elicitability,embrechts2018quantile}. The resulting functional is globally Lipschitz in the Wasserstein distance, which allows us to derive UCB-style regret bounds under the standard finite tabular assumptions.

\par For a deterministic policy \(\pi\in\Pi_{\mathrm{det}}\), a kernel \(P\), a stage
\(h\), and a state \(s\), define the remaining-return law $G_{h,s}^{\pi,P}
\coloneqq\mathcal L\!\left(\sum_{k=h}^{H-1} r_k(S_k,A_k)\,\middle|\,S_h=s,\pi, P\right)$, where the trajectory evolves under policy \(\pi\) and transition kernel \(P\) from stage \(h\) onward. The corresponding exact and buffered state--quantile values are $V_{q,h}^{\pi,P}(s)
\coloneqq
Q_q\!\left(G_{h,s}^{\pi,P}\right)
$, and$
V_{q,h}^{\pi,P,\beta}(s)
\coloneqq
Q_q^\beta\!\left(G_{h,s}^{\pi,P}\right)$.

\par Throughout this paper, we aim to maximize a quantile objective, defined as
the \(\tau\)-quantile of the return distribution for a fixed target level
\(\tau\in(0,1)\). Specifically, we consider
\begin{equation}
\label{eq:opti-goal}
\max_{\pi\in\Pi_{\mathrm{det}}}
V_{\tau,0}^{\pi,P^\star}(\bar s).
\end{equation}
where \(\bar s\in\mathcal S\) is a fixed designated initial state from which each episod starts. Let \(\pi^\star\) denote a maximizer of Equation~\eqref{eq:opti-goal}, and define
\(V_{\tau,0}^{\star}(\bar s)
\coloneqq V_{\tau,0}^{\pi^\star,P^\star}(\bar s) = \max_{\pi\in\Pi_{\mathrm{det}}}
V_{\tau,0}^{\pi,P^\star}(\bar s)\). 

The lower-buffered objective averages the quantile function over levels
immediately below the target level $\tau$. Consequently, whether the buffered
value agrees with the exact $\tau$-quantile depends on the local behavior of the
root quantile function to the left of $\tau$. For any fixed deterministic policy
in the finite-horizon tabular setting, the root return law has finite support,
and hence its quantile function is constant on a nonzero interval ending at
$\tau$. We quantify the smallest width of this left-side plateau uniformly over
all deterministic policies. This quantity determines how small the buffer must
be for the lower-buffered and exact root values to coincide and subsequently
appears in both the regret upper bound and the information-theoretic lower bound.

\begin{definition}[Root-level left-plateau threshold]
\label{def:root-plateau-threshold}
For each deterministic policy \(\pi\in\Pi_{\mathrm{det}}\) and the root return law
\(G_{0,\bar s}^{\pi,P^\star}\), the root-level left-plateau threshold at the target quantile \(\tau\) is defined
as $\rho_\tau\coloneqq\min_{\pi\in\Pi_{\mathrm{det}}}\big\{\tau-\lim_{y\uparrow V_{\tau,0}^{\pi,P^\star}(\bar s)}
\mathbb P\!\big(G_{0,\bar s}^{\pi,P^\star} \allowbreak \le y\big)\big\}$.
\end{definition}

\section{Optimistic Buffered Quantile Learning}\label{sec:UCB--BQRL}
This section develops the proposed optimistic learning framework and establishes
its main theoretical properties. We first introduce \UCBBQRL, which learns the
unknown transition model through confidence sets and plans using the
lower-buffered quantile objective, and then present \EVIBQ, the planning
procedure used to carry out the resulting optimistic re-planning step. We then
establish the correctness of this planning procedure and analyze the learning
performance of \UCBBQRL\ through a high-probability regret bound. Finally, we
study the fundamental limits of the framework through an information-theoretic
regret lower bound and the computational hardness of exact quantile evaluation.
\subsection{\UCBBQRL\space Algorithm}
\par First we introduce \UCBBQRL\space (Algorithm~\ref{alg:UCB--BQRL}), an optimistic
model-based algorithm for the finite-horizon \(\tau\)-quantile objective. This algorithm maintains a transition confidence set and stabilizes optimistic planning through a lower-buffered quantile criterion. We begin by fixing a designated start state $\bar s\in\mathcal S$. Each episode $t \in \{0,1,\dots,T-1\}$ starts at $S^t_0 = \bar s$, and within each episode, steps are indexed by $h\in\{0,\ldots,H-1\}$. Moreover, $N_h^t(s,a)$ and $N_h^t(s,a,s')$ denote visit and transition counts up to (but excluding) episode $t$.
\par Next, consider a fixed confidence level $\delta\in(0,1)$. Let $T$ denote the number of episodes, and $H$ the horizon length. We introduce a universal constant
\(
c\ \ge\ \frac{\max\!\{2,\sqrt{2\log\!( SATH(2^S-2) / {\delta})}\}}{\sqrt{\log\!(2SATH/ \delta )}},
\)
and define the confidence radius as $f_\delta(n)=c\sqrt{\frac{\log\!\frac{2SATH}{\delta}}{\max\{1,n\}}}$. Using this radius, we form an empirical confidence set 
\begin{equation}
\label{eq:confset}
\mathcal C^{t}_{\delta}\coloneqq \Bigl\{P:\,
\| P_h(\cdot|s,a)-\widehat P^{t}_h(\cdot|s,a)\|_1 \le f_\delta(N^{t}_h(s,a)), \ \forall s,a,h\Bigr\}.
\end{equation}

\par This set contains all transition kernels that are statistically plausible given the data observed up to episode $t$. On the global confidence event $\mathcal{E}_\delta$, which is defined as 
$\mathcal E_\delta \coloneqq \big\{ \big\|P_h^\star(\cdot\mid s,a)-\widehat P_h^t(\cdot\mid s,a)\big\|_1 \le f_\delta\!\big(N_h^t(s,a)\big), \forall t,h,s,a \big\} \allowbreak$, we have $P^\star\in\mathcal{C}^t_\delta$ simultaneously for all \(t\), \(h\), \(s\), and \(a\). 

\par The policy \(\pi^t\) is executed in the true MDP during episode \(t\). After observing episode \(t\), the learner updates the counts and empirical kernel to obtain \(\widehat P^{t+1}\), and forms the confidence region \(\mathcal C_\delta^{t+1}\). Given a buffer parameter \(\beta_{t+1}\in(0,\tau]\), it then plans optimistically over \(\mathcal C_\delta^{t+1}\) with respect to the lower-buffered root quantile objective:\((P^{t+1},\pi^{t+1})
\in
\arg\max_{P\in\mathcal C_\delta^{t+1}}
\max_{\pi\in\Pi_{\mathrm{det}}}
V_{\tau,0}^{\pi,P,\beta_{t+1}}(\bar s).\) This ``estimate \(\to\) buffer \(\to\) plan'' structure preserves optimism while replacing the
unstable point quantile by a globally Lipschitz buffered objective.

The exact planning oracle used by Algorithm~\ref{alg:UCB--BQRL} returns a deterministic policy together with a recursive policy label. Fix the ordering \(\mathcal S=\{s_1,\ldots,s_S\}\). The terminal label is denoted by \(L_H\) and contains no action. A nonterminal label has the form $L=(h,s,a,L_1,\ldots,L_S)$, where \(h\) is the stage, \(s\) is the current state, \(a\in\mathcal A\) is the action selected at \((h,s)\), and \(L_i\) is the continuation label used if the next state is \(s_i\). The label is constructed in the same backward recursion that constructs the
return law: whenever a law \(D\) at \((h,s)\) is generated from action \(a\) and
child frontier elements $(D_1,L_1),\ldots,(D_S,L_S) \in \mathfrak D_{h+1}^P(s_1)\times\cdots\times\mathfrak D_{h+1}^P(s_S)$, the associated label is set to $L=(h,s,a,L_1,\ldots,L_S)$. Thus \(L\) is not a new optimization variable; it is a recursive implementation record for the
policy generating \(D\). In implementation, \(L_1,\ldots,L_S\) are references to
previously constructed continuation labels, rather than copied subtrees. When a root label \(L^t\) is selected for episode \(t\), execution begins with \(L_0^t=L^t\). The symbol \(L_h^t\) denotes the continuation label reached at
stage \(h\) along the realized trajectory; in general, \(L_h^t\) is a descendant
of \(L^t\), not a new policy.
\vspace{5mm}
\begin{breakablealgorithm}
\caption{\UCBBQRL}
\label{alg:UCB--BQRL}
\begin{algorithmic}[1]
\State \textbf{Input:} target quantile level $\tau\!\in\!(0,1)$, confidence level $\delta\!\in\!(0,1)$, non-increasing buffer sequence \(\{\beta_t\}_{t\ge0}\subset(0,\tau]\). 
\State \textbf{Initialize:} for all \(h,s,a,s'\), set $N^0_{h}(s,a,s')\!\gets\!0$, and \(N_h^0(s,a)\gets0\), $\widehat P_h^0(s'\mid s,a)\gets \frac{1}{S}$, and complete simplex \(\mathcal C_\delta^0\).
\State Run Algorithm~\ref{alg:EVI-BQ} with
\((\mathcal C_\delta^0,\beta_0)\) to obtain
\((P^0,\pi^0,L^0,\widehat V_{\tau,0}^{0,\beta_0}(\bar s)),\) where \(L^0\) is the selected recursive policy label encoding \(\pi^0\).

\For{$t=0,1,\dots,T-1$} 
  \State \textit{Start:} $S_0^{t}\gets \bar s$ \label{alg:UCB--BQRL:L1}
  \State \textit{Roll out under the recursive policy label \(L^t\)}: set \(L_0^t\gets L^t\). For \(h=0,\ldots,H-1\), write the current label as $L_h^t=(h,S_h^t,a,L_1,\ldots,L_S)$. Set $A_h^t=a$, $S_{h+1}^t\sim P_h^\star(\cdot\mid S_h^t,A_h^t)$. If \(S_{h+1}^t=s_i\), set \(L_{h+1}^t=L_i\).

  \State \textit{Update per–step counts:} for each $h$, $N_h^{t+1}(s,a,s')=\sum_{i=0}^{t} \mathbf{1}\!\left\{\, S^i_{h}=s, A^i_{h}=a, S^i_{h+1}=s' \right\}$, and $N_h^{t+1}(s,a)=\sum_{i=0}^{t} \mathbf{1}\!\left\{\, S^i_{h}=s, A^i_{h}=a \right\}$.
  \State \textit{Update empirical model:} $\widehat P^{t+1}_h(s'|s,a)\gets\dfrac{N^{t+1}_h(s,a,s')}{\max\{1,N^{t+1}_h(s,a)\}}$
  
  \State \textit{Build confidence sets:} form $\mathcal C_\delta^{t+1}$ using Equation~\eqref{eq:confset}.
  \State \textit{Optimistic buffered re-planning:} run Algorithm~\ref{alg:EVI-BQ} with \((\mathcal C_\delta^{t+1},\beta_{t+1})\) to obtain $\big(P^{t+1},\pi^{t+1},L^{t+1},$ $\widehat V_{\tau,0}^{t+1,\beta_{t+1}}(\bar s)\big) \allowbreak$,
   where \(L^{t+1}\) is the selected recursive policy label encoding \(\pi^{t+1}\).\label{alg:UCB--BQRL-L10}
\EndFor
\end{algorithmic}
\end{breakablealgorithm}
\vspace{5mm}

The optimistic re-planning step in Line~\ref{alg:UCB--BQRL-L10} of
Algorithm~\ref{alg:UCB--BQRL} requires selecting a transition kernel
\(P^t\in\mathcal C_\delta^t\) together with a deterministic policy to be executed
in the next episode to solve the problem 
\(\max_{P\in\mathcal C_\delta^t}
\max_{\pi\in\Pi_{\mathrm{det}}}
V_{\tau,0}^{\pi,P,\beta_t}(\bar s)
\). Algorithm~\ref{alg:EVI-BQ} (EVI-BQ) implements this step by dynamic programming over
finite return laws. It is written as an extended value-iteration procedure, and for each fixed candidate model \(P\), it solves the corresponding fixed-model buffered planning problem by working backward from stage \(H\) to stage \(0\).

The main distinction from an ordinary scalar value iteration is that the
lower-buffered value of a policy is computed from the full return law generated
by that policy. Therefore, the recursion cannot store only one scalar value at
each state and stage. Instead, it stores a frontier of achievable return laws.
For a fixed model \(P\), stage \(h\), and state \(s\), consider starting the
process from \(S_h=s\) and then following a deterministic continuation policy
until the horizon. Each such continuation policy induces a probability
distribution over the remaining cumulative reward \(\sum_{k=h}^{H-1} r_k(S_k,A_k)
\). Algorithm~\ref{alg:EVI-BQ} stores these distributions in
\(\mathfrak D_h^P(s)\). Thus, \(\mathfrak D_h^P(s)\) is the collection of all
remaining-return laws that are achievable from \((h,s)\) under model \(P\). At the terminal stage \(H\), no rewards remain to be collected. Hence, the
remaining cumulative reward is the empty sum, which is equal to zero with
probability one. In the finite-law representation, this terminal return law is
written as \(\delta_0=\{(1,0)\}\), where the first coordinate \(1\) denotes probability mass one and the second
coordinate \(0\) denotes the return value. Thus, \(\delta_0\) is the point mass at zero.

We represent a finite return law \(D\) as a finite set of probability--value
pairs $D=\{(\lambda_m,x_m)\}_{m=1}^{M_D}$, where $\lambda_m\ge 0$, and $\sum_{m=1}^{M_D}\lambda_m=1$. Pairs with zero probability are removed, and pairs with the same value are merged
by summing their probabilities. We denote this operation by
\(\operatorname{Merge}(\cdot)\). For a finite law \(D\), let
\(x_{(1)}<\cdots<x_{(M)}\) be the sorted support values after merging, let
\(\lambda_{(j)}\) be the corresponding probabilities, and define $c_{j\ge1}=\sum_{\ell=1}^{j}\lambda_{(\ell)}$ and $c_{j=0}= 0$. For a finite law \(D\), the ordinary quantile function is a step function. In
particular, after sorting and merging the support values, for $j=1,\ldots,M$ the value
\(x_{(j)}\) is active for quantile levels \(u\) satisfying $c_{j-1}<u\le c_j$. Therefore, the lower-buffered quantile $Q_q^\beta(D)
=
\frac{1}{\ell_\beta(q)}
\int_{q-\ell_\beta(q)}^q Q_u(D)\,du$ can be computed by summing the contribution of each support value
\(x_{(j)}\) over the portion of the buffer interval
\([q-\ell_\beta(q),q]\) for which \(Q_u(D)=x_{(j)}\). The length of the overlap
between the buffer interval \([q-\ell_\beta(q),q]\) and the active quantile
interval \((c_{j-1},c_j]\) is $\left[
\min\{q,c_j\}
-
\max\{q-\ell_\beta(q),c_{j-1}\}
\right]_+$. Thus, for \(q\in(0,1]\), the lower-buffered quantile of \(D\) is evaluated as $Q_q^\beta(D)
=
\frac{1}{\ell_\beta(q)}
\sum_{j=1}^{M}
x_{(j)}
\left[
\min\{q,c_j\}
-
\max\{q-\ell_\beta(q),c_{j-1}\}
\right]_+$, where \(\ell_\beta(q)=\min\{\beta,q\}\) and \([z]_+=\max\{z,0\}\). As defined above, every frontier element has the form \((D,L)\), where \(D\) is
an achievable return law and \(L\) is a recursive label encoding one
deterministic continuation policy that generates \(D\). When a duplicate return
law is removed, one associated label is retained, so the duplicate-removal step
does not change the set of achievable return laws.

The backward recursion constructs return laws and recursive labels simultaneously. Suppose the frontiers \(\mathfrak D_{h+1}^P(s_1), \allowbreak \ldots,\mathfrak D_{h+1}^P(s_S)\) have already been constructed. To form a frontier element at \((h,s)\), the algorithm chooses an action \(a\) and one child frontier element \((D_i,L_i)\) for each possible next state \(s_i\). The one-step return law is then the mixture over next states:
with probability \(P_h(s_i\mid s,a)\), the process moves to \(s_i\), receives
the immediate reward \(r_h(s,a)\), and then follows the continuation law \(D_i\). Repeating this
operation for all actions and all continuation-law tuples gives the full
frontier \(\mathfrak D_h^P(s)\). The associated recursive label is $L=(h,s,a,L_1,\ldots,L_S)$.

After the backward recursion is complete for a fixed model \(P\), the root
frontier \(\mathfrak D_0^P(\bar s)\) contains all return laws achievable from the
initial state under deterministic policies. The algorithm selects the law
\(D_P^\star\) in this frontier with the largest lower-buffered value
\(Q_\tau^{\beta_t}(D_P^\star)\). Finally, after performing this fixed-model
optimization for every \(P\in\mathcal C_\delta^t\), the algorithm selects the
optimistic model \(P^t\), namely the model whose best achievable root law has the
largest lower-buffered value.

In Algorithm~\ref{alg:EVI-BQ}, Line~\ref{alg:evibq-terminal} initializes the
terminal frontiers with the zero-return law. Lines~\ref{alg:evibq-law}--
\ref{alg:evibq-label} perform the one-step backward extension described above. The duplicate-removal step in Line~\ref{alg:evibq-duplicate} is only a bookkeeping step; it does not change the set of achievable return laws, because at least one recursive label is retained for
each distinct law. Line~\ref{alg:evibq-root-law} solves the fixed-model buffered planning problem
for a given \(P\), and Line~\ref{alg:evibq-model} selects the optimistic model
from the confidence set.

Algorithm~\ref{alg:EVI-BQ} is an exact planning oracle. When
\(\mathcal C_\delta^t\) is infinite, the outer maximization over
\(P\in\mathcal C_\delta^t\) is understood as an exact optimization step, not as a
finite enumeration.
\vspace{5mm}
\begin{breakablealgorithm}
\caption{\EVIBQ: extended value iteration for buffered optimistic re-planning}
\label{alg:EVI-BQ}
\begin{algorithmic}[1]
\State \textbf{Input:} confidence set \(\mathcal C_\delta^t\), reward function
\(r\), target quantile \(\tau\in(0,1)\), buffer level
\(\beta_t\in(0,\tau]\), initial state \(\bar s\).
\State Fix an ordering \(\mathcal S=\{s_1,\ldots,s_S\}\).

\For{each fixed model \(P\in\mathcal C_\delta^t\)}
    \State Initialize the terminal frontiers:
    \(
    \mathfrak D_H^P(s)
    \gets
    \{(\delta_0,L_H)\},
    \qquad
    \forall s\in\mathcal S.
    \)
    \label{alg:evibq-terminal}

    \For{\(h=H-1,H-2,\ldots,0\)}
        \For{each state \(s\in\mathcal S\)}
            \State Set \(\mathfrak D_h^P(s)\gets\emptyset\).

            \For{each action \(a\in\mathcal A\)}
                \For{each tuple $\big((D_1,L_1),\ldots,(D_S,L_S)\big)
                \in
                \mathfrak D_{h+1}^P(s_1)
                \times\cdots\times
                \mathfrak D_{h+1}^P(s_S)$}
                    \State Construct
                    \(D \gets \operatorname{Merge} \left(\bigcup_{i=1}^{S}\left\{
                    \left(P_h(s_i\mid s,a)\lambda,\,r_h(s,a)+x \right):(\lambda,x)\in D_i \right\} \right).\)
                    \label{alg:evibq-law}

                    \State Create the recursive policy label
                    \(L\gets (h,s,a,L_1,\ldots,L_S).
                    \)
                    \label{alg:evibq-label}

                    \State Add \((D,L)\) to \(\mathfrak D_h^P(s)\).
                \EndFor
            \EndFor

            \State Remove duplicate laws from \(\mathfrak D_h^P(s)\), keeping
            one associated recursive label for each distinct law.
            \label{alg:evibq-duplicate}
        \EndFor
    \EndFor

    \State Select a best buffered root law under model \(P\):
    \((D_P^\star,L_P^\star)
    \in
    \arg\max_{(D,L)\in\mathfrak D_0^P(\bar s)}
    Q_\tau^{\beta_t}(D).
    \)
    \label{alg:evibq-root-law}

    \State Let \(\pi^P\) be the deterministic history-dependent policy obtained by traversing the selected root label \(L_P^\star\), and set $\widehat V_{\tau,0}^{P,\beta_t}(\bar s)\gets Q_\tau^{\beta_t}(D_P^\star)$. \label{alg:evibq-fixed-model-value}

\EndFor

\State Select the optimistic model: $P^t
\in
\arg\max_{P\in\mathcal C_\delta^t}
\widehat V_{\tau,0}^{P,\beta_t}(\bar s)$. \label{alg:evibq-model}

\State Set $L^t\gets L_{P^t}^\star,
\qquad
\pi^t\gets \pi^{P^t},
\qquad
\widehat V_{\tau,0}^{\,t,\beta_t}(\bar s)
\gets
\widehat V_{\tau,0}^{P^t,\beta_t}(\bar s)$.
\label{alg:evibq-policy}
\State \textbf{Output:} optimistic model \(P^t\), deterministic policy
\(\pi^t\), recursive label \(L^t\), and buffered optimistic
root value \(\widehat V_{\tau,0}^{\,t,\beta_t}(\bar s)\).
\end{algorithmic}
\end{breakablealgorithm}
\vspace{5mm}
The construction above explains how \EVIBQ\ builds return-law frontiers and
selects a root law, but it remains to verify that this procedure indeed solves
the optimistic planning problem used by \UCBBQRL. In particular, we need to
establish that the backward recursion captures every return law that can be
generated by a deterministic history-dependent policy under a given transition
model. If this is true, then maximizing the lower-buffered quantile over the
root frontier is equivalent to maximizing over all such policies, and the final
maximization over the confidence set yields the desired optimistic model--policy
pair. The following proposition establishes precisely this connection between
the frontier recursion and the optimistic re-planning problem.
\begin{proposition}[Correctness of \EVIBQ]
\label{prop:evibq}
Fix an episode \(t\), a target quantile level \(\tau\in(0,1)\), and a buffer
\(\beta_t\in(0,\tau]\). Suppose that \(\mathcal C_\delta^t\) is nonempty and
that the maximizers in Algorithm~\ref{alg:EVI-BQ} exist. Then the returned pair
\((P^t,\pi^t)\) satisfies $P^t\in\mathcal C_\delta^t$ and $V_{\tau,0}^{\pi^t,P^t,\beta_t}(\bar s)
= \max_{P\in\mathcal C_\delta^t}
\max_{\pi\in\Pi_{\mathrm{det}}}
V_{\tau,0}^{\pi,P,\beta_t}(\bar s)$. Consequently, Algorithm~\ref{alg:EVI-BQ} exactly implements the optimistic
buffered re-planning step required by Algorithm~\ref{alg:UCB--BQRL}.
\end{proposition}

% \begin{remark}[Computationally tractable approximations]
% \label{rem:evibq-tractable-approximation}
% Algorithm~\ref{alg:EVI-BQ} is stated as an exact optimistic planning oracle. The
% outer maximization over $\mathcal C_\delta^t$ is not claimed to be computationally efficient in full generality. Its role is
% to specify the exact policy--model pair required by the regret analysis of
% Algorithm~\ref{alg:UCB--BQRL}. A computationally tractable implementation can instead replace
% \(\mathcal C_\delta^t\) by a finite candidate set $\mathcal G^t
% \subseteq
% \mathcal C_\delta^t$, constructed, for example, by discretizing the confidence set, sampling plausible
% transition kernels, or selecting representative extreme-point models. The
% planner then returns the best model in \(\mathcal G^t\). This finite-candidate
% version is an approximate optimistic planner. If its planning error satisfies $\max_{P\in\mathcal C_\delta^t}
% \max_{\pi\in\Pi_{\mathrm{det}}}
% \big(V_{\tau,0}^{\pi,P,\beta_t}(\bar s)
% -
% V_{\tau,0}^{\pi^t,P^t,\beta_t}(\bar s)\big)
% \le
% \varepsilon_{\mathrm{plan}}^t$, then the regret bound acquires an additional additive term $\sum_{t=0}^{T-1}\varepsilon_{\mathrm{plan}}^t$. Consequently, the sublinear regret guarantee of Theorem~\ref{thm:UCB-BQRL-finite}
% is preserved only when this cumulative approximation error is sublinear in
% \(T\). In particular, if \(\varepsilon_{\mathrm{plan}}^t\) is bounded away from
% zero uniformly in \(t\), then the additional term is linear in \(T\), and the
% sublinear guarantee is no longer obtained from the stated analysis.
% \end{remark}

\subsection{Theoretical Properties}
\subsubsection{High-probability Regret Upper Bound}
\par Having specified the algorithm, we now turn to its performance analysis. Our objective is to measure how much reward is lost by following Algorithm~\ref{alg:UCB--BQRL} compared to the optimal $\tau$–quantile policy in the true environment. 
This gap is captured by the notion of \emph{quantile regret}:
\begin{equation}
\label{eq:regret}
\Reg_{\tau}(T)
\;=\;
\sum_{t=0}^{T-1}
\Bigl(
V^{\pi^\star,P^\star}_{\tau,0}(\bar s)-
V^{\pi^{t},P^\star}_{\tau,0}(\bar s)
\Bigr),
\end{equation}
where $\pi^\star$ is the optimal $\tau$–quantile policy under the true kernel $P^\star$. Note that $\tau$ is fixed in the regret definition, while other quantile levels $q \in (0,1)$ appear internally only as evaluation arguments of \(Q_q\) and
\(Q_q^\beta\). Here, for \(\beta\in(0,\tau]\), we also define the root-level buffering gap $\Delta_\beta \coloneqq \sup_{\pi\in\Pi_{\mathrm{det}}}
\left|V_{\tau,0}^{\pi,P^\star}(\bar s)-V_{\tau,0}^{\pi,P^\star,\beta}(\bar s)
\right|$.

\begin{theorem}[High-probability quantile regret of \UCBBQRL]
\label{thm:UCB-BQRL-finite}
Let \(\{\beta_t\}_{t\ge0}\subset(0,\tau]\) be non-increasing. Then, for
\UCBBQRL\space with confidence radii in
Equation~\eqref{eq:confset}, with probability at least \(1-2\delta\), $\Reg_\tau(T)
\le\;
2\sum_{t=0}^{T-1}\Delta_{\beta_t} +
\frac{c}{\beta_{T-1}}
\sqrt{SATH^4\log\!\frac{2SATH}{\delta}}
+
\frac{c}{\beta_{T-1}}
\sqrt{
\frac{TH^4}{2}\log\!\frac{2SATH}{\delta}
\log\!\frac{1}{\delta}
}$.
\end{theorem}

The regret bound in Theorem~\ref{thm:UCB-BQRL-finite} has two components. The first
term, \( 2 \sum_{t=0}^{T-1}\Delta_{\beta_t},\)
is the price of evaluating exact \(\tau\)-quantile regret while planning with a
lower-buffered criterion. By Lemma~\ref{lem:buffer-gap}, each summand \(\Delta_{\beta_t}\) is zero whenever \(\beta_t\le\rho_\tau\). Hence, under any buffer schedule that eventually falls below \(\rho_\tau\), only a finite initial segment of the buffering cost can be nonzero. The remaining two terms are the statistical transition-estimation error. The factor \(1/\beta_{T-1}\) comes from the global Wasserstein Lipschitz constant of the lower-buffered quantile, while the horizon dependence comes from summing the maximum number of future rewards that can be affected by transition-model mismatches across the episode.

\begin{corollary}[Regret under the logarithmic buffer]
\label{cor:buffered-UCB-BQrl-log-rate}
For each $t$ choose \(\beta_t = \frac{\tau} {\log(\mathrm e + t)},\) and define \(K_\tau \coloneqq \left\lceil \exp\!\left(\frac{\tau}{\rho_\tau}\right)-\mathrm e \right\rceil\). Then, under the conditions of Theorem~\ref{thm:UCB-BQRL-finite}, with probability at least \(1-2\delta\), $\Reg_\tau(T)\le\;2H K_\tau+ \allowbreak \frac{c\log(\mathrm e+T-1)}{\tau}\sqrt{SA TH^4\log\!\frac{2SATH}{\delta}} + \allowbreak\frac{c\log(\mathrm e+T-1)}{\tau}\sqrt{\frac{TH^4}{2}\log\!\frac{2SATH}{\delta}\log\!\frac{1}{\delta}} \allowbreak$. Consequently,
\(\Reg_\tau(T) = \widetilde{\mathcal O}\!\left(
\frac{1}{\tau}\sqrt{SATH^4}
\right) + O(K_\tau H)\).
\end{corollary}
Theorem~\ref{thm:UCB-BQRL-finite} provides a regret bound for any
non-increasing buffer sequence \(\{\beta_t\}_{t\ge0}\). 
Corollary~\ref{cor:buffered-UCB-BQrl-log-rate} specializes this result to the
particular logarithmic choice
\(\beta_t=\tau/\log(\mathrm e+t)\). Under this schedule, the buffer decreases
over time and eventually becomes smaller than the root-level plateau width
\(\rho_\tau\). From that point onward, the buffered and exact
\(\tau\)-quantile values coincide, so the buffering gap vanishes.
The resulting bound therefore consists of a finite instance-dependent
buffering term \(O(HK_\tau)\) and a statistical learning term of order
\(\widetilde{\mathcal O}\!\left(\frac{H^2}{\tau}\sqrt{SAT}\right)\).
In particular, this gives a concrete sublinear regret guarantee for the logarithmic
buffer schedule.
\subsubsection{Information-Theoretic Lower Bound for Quantile Regret}\label{sec:IT-lower-bound}

We complement Theorem~\ref{thm:UCB-BQRL-finite} with the presentation of an information-theoretic lower bound. This lower bound is independent of any algorithmic buffer schedule and is stated in terms of the instance dependent parameter \(\rho_\tau\), the root-level left-plateau length at the target quantile as defined in Definition~\ref{def:root-plateau-threshold}.

\begin{theorem}[Information-theoretic lower bound for exact quantile regret]
\label{thm:UCB-BQ-IT_lower}
Fix a target quantile level \(\tau\in(0,1)\), and assume \(A\ge2\) and \(H\ge2\). Then there exist constants \(c_0,c_1>0\), depending at most on the fixed target level \(\tau\), such that for every \(\rho\in(0,\frac{\min\{\tau,1-\tau\}}{8}]\), every
\(T\ge c_0\,\frac{A}{\rho^2}\), and every possibly randomized learning algorithm that, at each episode \(t\),
outputs a deterministic history-dependent policy
\(\pi^t\in\Pi_{\mathrm{det}}\), there exists a finite-horizon episodic MDP \(\mathcal M=(\mathcal S,\mathcal A,H,P^\star,r)
\) with deterministic rewards \(r_h(s,a)\in[0,1]\),
\(
|\mathcal S|=2\), \(|\mathcal A|=A\), such that the root-level plateau constant of this instance satisfies $\rho_\tau=\rho$, and 
$\mathbb E\!\left[\Reg_\tau(T)\right]
\ge c_1\,\frac{H} {\rho_\tau}\sqrt{AT}$.
\end{theorem}

Theorem~\ref{thm:UCB-BQ-IT_lower} identifies an intrinsic statistical difficulty of learning under an exact quantile objective. In particular, the lower bound holds even for a two-state MDP with deterministic rewards and is independent of the buffer sequence or planning procedure used by the learner.
Hence, the dependence on the instance parameter $\rho_\tau$ is not an artifact
of the lower-buffered approximation introduced in \UCBBQRL. By Definition~\ref{def:root-plateau-threshold}, $\rho_\tau$ measures, in
probability, the smallest gap between the target level $\tau$ and the cumulative
probability immediately below the corresponding $\tau$-quantile, taken over all
deterministic policies.When $\rho_\tau$ is small, only a small change in transition probabilities is
needed to move this cumulative probability across the level $\tau$, which can
cause the exact $\tau$-quantile to jump to a different return value. In the hard instances used to prove the theorem, such a change can switch the
$\tau$-quantile between $0$ and $H-1$. Distinguishing such instances
therefore requires greater statistical precision, which is reflected by the
factor $1/\rho_\tau$ in the lower bound
$\Omega\!\left(\frac{H}{\rho_\tau}\sqrt{AT}\right)$. The theorem also shows that the $\sqrt{T}$ dependence is unavoidable in general. Together with
Corollary~\ref{cor:buffered-UCB-BQrl-log-rate}, these results establish that
sublinear $\sqrt{T}$-type learning is achievable while also demonstrating that
the local geometry of the return distribution near the target quantile
fundamentally governs the difficulty of the learning problem.

\subsubsection{Computational hardness of exact quantile evaluation}\label{sec:hardness}
The preceding regret analysis separates statistical learnability from computational
tractability: Algorithm~\ref{alg:EVI-BQ} is used as an exact planning oracle, but
the existence of such an oracle does not imply that exact quantile planning is
efficient in general. Establishing a computational hardness result is therefore
important for clarifying the role of the exact tabular planner.

Exact computation of a scalar objective immediately yields its value-comparison
problem: an exact evaluator for \(Q_\tau(R^\pi)\) or \(Q_\tau^\beta(R^\pi)\)
can decide whether the corresponding value is at most any given threshold \(c\). Thus, the hardness of the associated threshold evaluation problem rules out a generic polynomial-time exact evaluation procedure. For a fixed deterministic policy \(\pi\), let \(R^\pi\) denote the total return
generated from the initial state. The ordinary QMDP evaluation problem asks for
the value $V_{\tau,0}^{\pi,P}(\bar s)=Q_\tau(R^\pi)$, and its threshold version asks whether $Q_\tau(R^\pi)\le c$. Similarly, the lower-buffered evaluation problem asks for $V_{\tau,0}^{\pi,P,\beta}(\bar s)=Q_\tau^\beta(R^\pi)$, and its threshold version asks whether $Q_\tau^\beta(R^\pi)\le c$.
Using this, Theorem~\ref{thm:quantile-eval-hardness} establishes that both exact
point-quantile evaluation and exact lower-buffered quantile evaluation are
PP-hard under polynomial-time Turing reductions, even for a fixed policy in a
two-state, one-action finite-horizon MDP.

\begin{theorem}[Hardness of exact quantile and lower-buffered quantile evaluation]
\label{thm:quantile-eval-hardness}
The following decision problems are PP-hard under polynomial-time Turing
reductions:
\begin{enumerate}
    \item exact finite-horizon QMDP evaluation, i.e., deciding whether
    \(Q_\tau(R^\pi)\le c\) for the return \(R^\pi\) of a fixed deterministic
    policy;
    \item exact lower-buffered quantile evaluation, i.e., deciding whether
    \(Q_\tau^\beta(R^\pi)\le c\) for the return \(R^\pi\) of a fixed
    deterministic policy.
\end{enumerate}
The hardness holds even for finite-horizon MDPs with two states, one action,
binary transition probabilities, and deterministic rational rewards in
\([0,1]\). Consequently, the corresponding exact planning problems are also
PP-hard under polynomial-time Turing
reductions.
\end{theorem}

\begin{remark}[Representation size versus computational hardness]
Theorem~\ref{thm:quantile-eval-hardness} is a computational hardness result
for exact evaluation and does not rely on the explicit probability--value
frontier representation used by Algorithm~\ref{alg:EVI-BQ}. Moreover, the backward recursion of Algorithm~\ref{alg:EVI-BQ}
explicitly constructs and stores distinct achievable return laws, and both the
number of these laws and the sizes of their supports can grow exponentially in
some instances. Theorem \ref{thm:quantile-eval-hardness} shows a stronger point: even fixed-policy
exact quantile evaluation is already PP-hard in general. Thus, a generic
polynomial-time exact planner is not expected without additional structure.
\end{remark}

\section{Numerical Experiment}\label{sec:experiments}
We evaluate a practical implementation of \UCBBQRL\space on a finite-horizon
asset-selling problem. The asset-selling problem is a classical optimal-stopping
problem in which, at each decision period, a seller observes a current offer and
must decide whether to accept it or reject it and wait for a new random offer
\cite{seierstad1992reservation,rosenfield1983optimal}. This transparent
stopping problem allows us to examine two questions. First, how does the quality
of the policies learned by \UCBBQRL\space compare with that of standard
expectation-based and model-free reinforcement-learning methods? Second, how
does the target quantile $\tau$ affect the resulting stopping rule? Following
the QMDP perspective that a quantile-sensitive method should be assessed both
through its return criterion and through the statewise decisions it induces
\cite{li2022quantile}, we report cumulative policy-performance measures and
the final policies. Because the unrestricted \EVIBQ\space oracle can require
rapidly growing return-law frontiers, the numerical study does not execute the
exact oracle analyzed in the theoretical results. Instead, we use a
computational approximation of its return-law recursion, with the
specific candidate-model, frontier-capping, and pruning rules described below.

\subsection{Model Formulation}

We consider a finite-horizon, time-homogeneous asset-selling MDP with
$H=10$ decision periods. The state space consists of $25$ possible offer
states,
$\mathcal S_{\mathrm{off}}=\{0,\ldots,24\}$, together with an absorbing
post-sale state $s_{\mathrm{term}}$. Every episode starts from the fixed initial
offer $\bar s=5$. At each offer state, the decision maker chooses between two
actions,
$\mathcal A=\{\mathrm{Stop},\mathrm{Continue}\}$. If the decision maker chooses $\mathrm{Stop}$ at offer
$s\in\mathcal S_{\mathrm{off}}$, the asset is sold and the process moves to
$s_{\mathrm{term}}$. The corresponding reward is
$r(s,\mathrm{Stop})=s/24$, where the division by $24$ normalizes the offer
values to the interval $[0,1]$. Once the asset is sold, the process remains in
$s_{\mathrm{term}}$ and receives zero reward for the remainder of the horizon. If the decision maker chooses $\mathrm{Continue}$, the current offer is rejected
and the immediate reward is zero. A new offer is then drawn independently from the $25$ offer states. In the true environment, the next offer is uniformly distributed, so $r(s,\mathrm{Continue})=0$, $P^\star(s'\mid s,\mathrm{Continue})=\frac{1}{25}$, $s'\in\mathcal S_{\mathrm{off}}$. Thus, conditional on continuing, the distribution of the next offer is
independent of the current offer. Since the only positive reward is obtained
when the asset is sold, the total episodic return is the normalized accepted
offer; if no offer is accepted within the horizon, the return is zero.

At the final decision period, choosing $\mathrm{Continue}$ yields zero immediate
reward and moves the process to a new offer only after the last available
decision has been made. That new offer therefore cannot be accepted within the
episode. In contrast, choosing $\mathrm{Stop}$ yields the nonnegative reward
$s/24$. Hence, $\mathrm{Stop}$ weakly dominates $\mathrm{Continue}$ in the
final period, and restricting the final-period action to $\mathrm{Stop}$ is
without loss of optimality.

In the learning problem, the reward function and the structural form of the
transitions are known: stopping leads to the absorbing state, continuing leads
to one of the offer states, and the post-sale state is absorbing. The
probabilities associated with the possible continuation transitions are not
provided to the learner and must be estimated from observed trajectories. The
uniform probabilities above specify the true data-generating environment used
in the experiment.

\subsection{Implementation and Experimental Design}

The numerical implementation follows the same confidence-set and buffered
planning principles as \UCBBQRL, but it replaces the exact optimization in
\EVIBQ\ with a computationally tractable planner. Before each episode, the
algorithm updates the empirical transition kernel from all previously observed
transitions and constructs row-wise $\ell_1$ confidence radii using the form in
Equation~\eqref{eq:confset}. We set $\delta=0.05$. In the theoretical
confidence radius, the constant $c$ is chosen to guarantee simultaneous
confidence-set containment. In the numerical implementation, we instead use a
tunable multiplier $c_{\mathrm{conf}}$, whose value is selected by the
validation procedure described below. We replan before every episode and use
the logarithmic buffer schedule
$\beta_t=\tau/\log(\mathrm e+t)$ from
Corollary~\ref{cor:buffered-UCB-BQrl-log-rate}.

The practical planner makes two approximations relative to \EVIBQ. First,
instead of optimizing over every transition kernel in the confidence set
$\mathcal C_\delta^t$, it considers six representative candidate kernels. These
consist of the empirical kernel, one directed perturbation that shifts
probability mass toward a favorable successor while remaining within the
confidence radius, and four randomized perturbations that also satisfy the
row-wise confidence radius and the transition structure described in the
preceding subsection. The planner evaluates all six candidates and selects the
one with the largest lower-buffered value at the initial state. Thus, the first
approximation replaces the full confidence-set optimization by a finite search
over representative models.

Second, for each candidate kernel, the implementation does not construct the
full frontier of achievable return laws used by \EVIBQ. Instead, the backward
recursion retains a single continuation-return law at each stage--state pair.
Given the retained laws at stage $h+1$, the planner constructs the return law
associated with each available action at $(h,s)$ and selects the action with the
largest lower-buffered $\tau$-quantile, using the mean return only to break
ties. The selected action and its return law are then retained for the preceding
stage. Consequently, the implemented policy is a deterministic
stage-dependent Markov policy, whereas \EVIBQ\ searches over the larger class
of deterministic history-dependent policies represented by its complete
return-law frontiers. Algorithm~\ref{alg:practical-bqrl-planner} summarizes this
practical planning procedure.
\vspace{5mm}
\begin{breakablealgorithm}
\caption{Practical buffered optimistic planner used in the numerical experiment}
\label{alg:practical-bqrl-planner}
\begin{algorithmic}[1]

\State \textbf{Input:} empirical kernel $\widehat P^t$, counts
$N_h^t(s,a)$, target quantile $\tau$, buffer $\beta_t$, confidence level
$\delta$, multiplier $c_{\mathrm{conf}}$, reward function $r$, and initial
state $\bar s$.

\State For each $(h,s,a)$, define the practical confidence radius $\varepsilon_h^t(s,a)
=
c_{\mathrm{conf}}
\sqrt{
\frac{
\log(2SATH/\delta)
}{
\max\{1,N_h^t(s,a)\}
}
}$.

\State Construct the candidate-model set $\mathcal G^t
=
\left\{
\widehat P^t,\,
P_{\mathrm{dir}}^t,\,
P_{\mathrm{rand},1}^t,\ldots,
P_{\mathrm{rand},4}^t
\right\}$, where every candidate $P\in\mathcal G^t$ satisfies $\left\|
P_h(\cdot\mid s,a)
-
\widehat P_h^t(\cdot\mid s,a)
\right\|_1
\le
\varepsilon_h^t(s,a)$ for every transition row with unknown probabilities and satisfies the known
transition structure of the asset-selling model.

\For{each $P\in\mathcal G^t$}

    \State Initialize, for every $s\in\mathcal S$, $D_H^P(s)\gets\delta_0$.

    \For{$h=H-1,H-2,\ldots,0$}

        \For{each $s\in\mathcal S$}

            \For{each $a\in\mathcal A$}

                \State Construct the return law generated by taking action
                $a$ at $(h,s)$ and then using the retained continuation law
                $D_{h+1}^P(s')$ after each possible next state $s'$: $D_{h}^{P,a}(s)
                \gets
                \operatorname{Merge}
                \left(
                \bigcup_{s'\in\mathcal S}
                \left\{
                \left(
                P_h(s'\mid s,a)\lambda,\,
                r_h(s,a)+x
                \right):
                (\lambda,x)\in D_{h+1}^P(s')
                \right\}
                \right)$.

                \State Compute its lower-buffered value $B_h^{P,a}(s)
                \gets
                Q_\tau^{\beta_t}
                \left(
                D_h^{P,a}(s)
                \right)$, and its mean return $M_h^{P,a}(s)
                \gets
                \sum_{(\lambda,x)\in D_h^{P,a}(s)}
                \lambda x$.

            \EndFor

            \State Select $a_h^P(s)
            \in
            \arg\max_{a\in\mathcal A}
            B_h^{P,a}(s)$, breaking ties by the largest $M_h^{P,a}(s)$.

            \State Retain only the selected continuation law: $D_h^P(s)
            \gets
            D_h^{P,a_h^P(s)}(s)$.

        \EndFor
    \EndFor

    \State Define the resulting stage-dependent Markov policy by $\pi_h^P(s)\gets a_h^P(s)$, $h=0,\ldots,H-1$, and its buffered root value by $\widehat V_{\tau,0}^{P,\beta_t}(\bar s)
    \gets Q_\tau^{\beta_t}\!\left(D_0^P(\bar s)\right)$.

\EndFor

\State Select the candidate model $P^t
\in
\arg\max_{P\in\mathcal G^t}
\widehat V_{\tau,0}^{P,\beta_t}(\bar s)$, and set $\pi^t\gets\pi^{P^t}$.

\State \textbf{Output:}
$P^t$, $\pi^t$, and
$\widehat V_{\tau,0}^{P^t,\beta_t}(\bar s)$.

\end{algorithmic}
\end{breakablealgorithm}
\vspace{5mm}
We compare \UCBBQRL\ with six established reinforcement-learning benchmarks:
the model-based optimistic algorithm UCBVI \cite{azar2017minimax};
$Q$-learning with an $\varepsilon$-greedy behavior policy
\cite{watkins1992q,sutton1995generalization}; the on-policy
temporal-difference algorithm SARSA \cite{rummery1994line}; Thompson sampling
\cite{william1933likelihood}; and the policy-gradient methods TRPO
\cite{schulman2015trust} and PPO \cite{schulman2017proximal}. This comparison
set spans model-based optimism, off-policy and on-policy temporal-difference
learning, posterior-sampling exploration, and policy-gradient methods.

All methods interact with the same finite-horizon, time-homogeneous
asset-selling environment described above. The deterministic reward function is
available directly to the two model-based methods, \UCBBQRL\ and UCBVI, which
learn only the unknown transition probabilities. The model-free methods are not
given the reward function as a model input; instead, they observe the realized
reward after each interaction and update their decision rules from the observed
state, action, reward, and next state. Their training updates use a discount
factor of $\gamma=0.99$. This discounting is used only in their learning
updates. All reported policy values are evaluated using the undiscounted
finite-horizon return defined in the preceding subsection, so the evaluation
criterion is common across all methods.

The target quantile $\tau$ affects the training objective of \UCBBQRL\ through
its lower-buffered quantile criterion. In contrast, the training objectives of
UCBVI, $Q$-learning, SARSA, Thompson sampling, PPO, and TRPO do not depend on
$\tau$. For these benchmark methods, $\tau$ enters only when a learned policy is
scored using the quantile-based validation and evaluation criteria described
below. Accordingly, the benchmark algorithms are tuned once using
$\tau=0.5$, whereas \UCBBQRL\ is tuned separately for
$\tau\in\{0.1,0.5,0.9\}$.

All structural choices are fixed before hyperparameter tuning. Each method is
assigned a $16$-point logarithmic grid over one dominant parameter:
$c_{\mathrm{conf}}$ for \UCBBQRL, the exploration-bonus multiplier for UCBVI, the learning rate for $Q$-learning and SARSA, the posterior-noise scale for Thompson sampling, the policy learning rate for PPO, and the trust-region radius for TRPO. We use successive halving with $\eta=2$ and the schedule \((16,1000)\rightarrow(8,2000)\rightarrow(4,4000)\rightarrow(2,8000),\) where each pair gives the number of active configurations and the number of training episodes per validation seed. Validation uses seeds $1000$, $1001$, and $1002$. For \UCBBQRL, configurations are ranked using the cumulative $\tau$-quantile policy-evaluation gap at the corresponding target level. For the $\tau$-independent benchmarks, the same criterion is evaluated at
$\tau=0.5$ for hyperparameter selection. Mean reward over the final $10\%$ of episodes is used only as a tie-breaker. Final evaluation uses $T=2000$ episodes and the held-out seeds $42$, $10042$, and $20042$.

To compare the quality of the policies produced during training, we evaluate a
deterministic policy induced by the current state of each method. For methods
that explicitly maintain a deterministic policy, this is the current planned
policy; when training uses exploratory or stochastic action selection, the
exploration randomization is removed for evaluation. For example, an
$\varepsilon$-greedy method is evaluated using its greedy action rather than an
$\varepsilon$-randomized action. We denote the resulting policy at episode $t$
by $\bar\pi^t$. This evaluation therefore measures the quality of the policy
learned by each method and should not be interpreted as the realized online
regret generated by its exploratory behavior during training.

For the quantile-dependent comparisons, we report the cumulative numerical $\tau$-quantile policy gap $\widehat{\mathcal G}_{\tau}(k)
=\sum_{t=1}^{k}\left[\widehat V_{\tau,0}^{\star}(\bar s)-V_{\tau,0}^{\bar\pi^t,P^\star}(\bar s)
\right]_{+}$, $k=1,\ldots,T$.
For each evaluated policy $\bar\pi^t$, the value
$V_{\tau,0}^{\bar\pi^t,P^\star}(\bar s)$ is obtained from its finite-horizon return distribution under the true transition kernel. The reference value $\widehat V_{\tau,0}^{\star}(\bar s)$ is computed once for each $\tau$ using
the true model and is used identically for all algorithms.

Because computing the unrestricted optimal quantile reference is itself
computationally difficult, $\widehat V_{\tau,0}^{\star}(\bar s)$ is obtained
from a approximated version of the return-law frontier recursion. Whenever two
constructed return laws are identical, only one copy is retained. Return
values are rounded to six decimal places, and return laws that are
first-order stochastically dominated are removed when possible. At most $16$
return laws are retained at each frontier. In addition, when combining the
child frontiers would require evaluating more than $4000$ combinations, the
implementation uses the best retained continuation from each child frontier
instead of enumerating the full Cartesian product. These restrictions keep the
reference computation tractable, but they also mean that
$\widehat V_{\tau,0}^{\star}(\bar s)$ is a numerical benchmark rather than the
exact value $V_{\tau,0}^{\star}(\bar s)$.

The positive-part operator $[\,\cdot\,]_+$ in the policy-gap definition is used
for this reason. Because the reference value is numerical, an evaluated policy
can occasionally have a computed quantile value slightly larger than
$\widehat V_{\tau,0}^{\star}(\bar s)$. In such a case, the corresponding gap is
set to zero rather than reported as a negative loss. Thus, the resulting
quantity is a policy-evaluation measure relative to a common numerical
reference and is distinct from the exact quantile regret defined in
Equation~\eqref{eq:regret}.

All reported curves are averaged over the three held-out seeds. At each
episode, the shaded region is
$\overline G_k\pm1.96\,\widehat\sigma_k/\sqrt{3}$, where
$\overline G_k$ and $\widehat\sigma_k$ are the sample mean and sample standard
deviation across the three runs. Because only three held-out seeds are
available, we use these bands only as descriptive summaries of cross-seed
variability and do not interpret them as reliable inferential confidence
intervals.

\subsection{Learning Results}

Figures~\ref{fig:asset-selling-regret-tau01}--%
\ref{fig:asset-selling-regret-tau09} report two complementary measures of
learning performance. Figure~\ref{fig:asset-selling-regret-tau01} evaluates
performance under the conventional expected-return objective, whereas
Figures~\ref{fig:asset-selling-regret-tau05} and
\ref{fig:asset-selling-regret-tau09} evaluate performance under the
quantile-dependent criterion introduced in the preceding subsection. The
figures therefore should not be interpreted as plotting the same performance
measure at three different values of $\tau$.

To define the expected-return measure used in
Figure~\ref{fig:asset-selling-regret-tau01}, let
$J^{\pi,P^\star}(\bar s)$ denote the expected undiscounted finite-horizon
return of policy $\pi$ in the true environment $J^{\pi,P^\star}(\bar s)
\coloneqq
\mathbb E_{\pi,P^\star}
\left[
\sum_{h=0}^{H-1} r_h(S_h,A_h)
\,\middle|\, S_0=\bar s\right]$, and let $J^\star(\bar s)\coloneqq
\max_{\pi}J^{\pi,P^\star}(\bar s)$.
For the evaluated policies $\{\bar\pi^t\}$ defined above, we report the cumulative expected-return regret $\mathcal R_{\mathrm E}(k)
\coloneqq
\sum_{t=1}^{k}
\left[
J^\star(\bar s)
-
J^{\bar\pi^t,P^\star}(\bar s)
\right]$, $k=1,\ldots,T$.

Figure~\ref{fig:asset-selling-regret-tau01} reports
$\mathcal R_{\mathrm E}(k)$ when \UCBBQRL\ is trained with the conservative
target $\tau=0.1$. At episode $2000$, UCBVI has the smallest mean cumulative
expected-return regret, $15.73$, followed by \UCBBQRL\ at $229.00$. SARSA,
which has the smallest value among the remaining methods, reaches $586.19$.
Thus, UCBVI performs best under the expected-return criterion, as expected for
an algorithm designed around optimistic expected-value learning. Nevertheless,
\UCBBQRL\ remains the second-best method on this measure despite being trained
to optimize a lower-tail quantile rather than expected return. This result
suggests that, in this instance, emphasizing conservative quantile performance
does not require \UCBBQRL\ to sacrifice expected-return performance to the
extent observed for the other benchmark methods.

\begin{figure}[t]
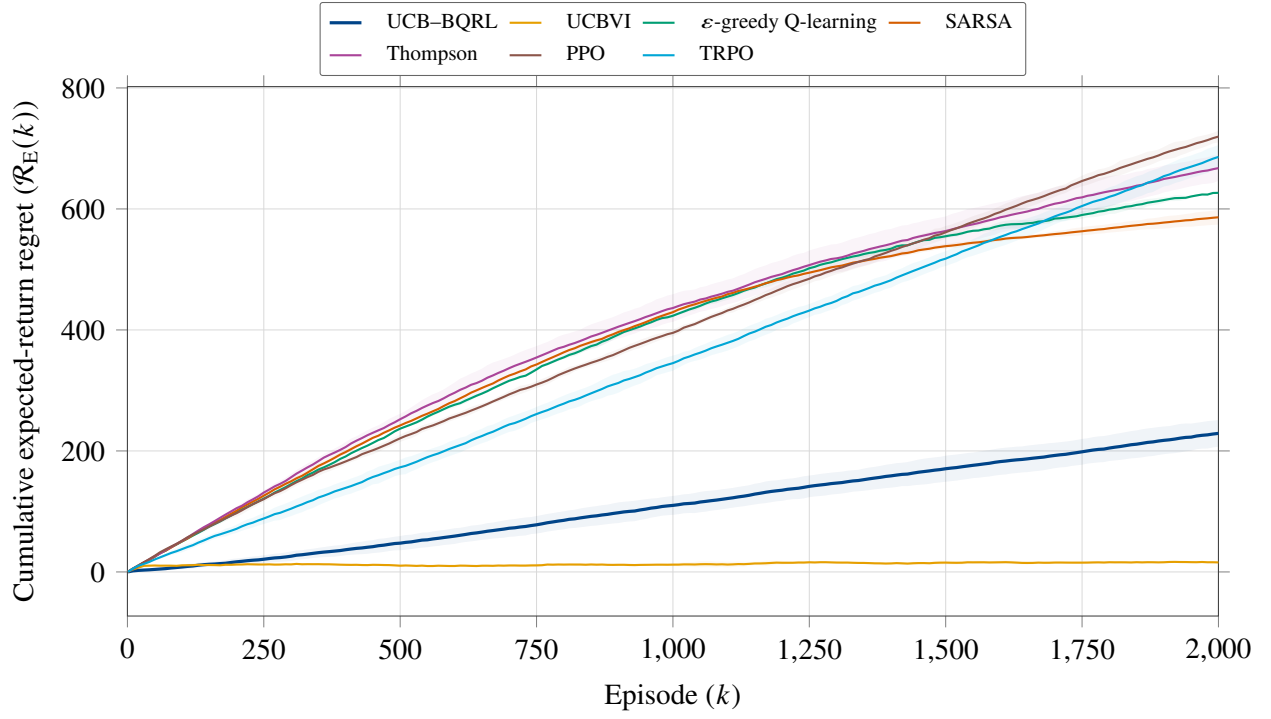

\centering
\assetregretplot[
Cumulative expected-return regret
($\mathcal R_{\mathrm E}(k)$)
]
{\QMDPDataDir/AssetSelling_tau0p1_plot.csv}
\caption{Cumulative expected-return regret
$\mathcal R_{\mathrm E}(k)$ in the asset-selling experiment when
\UCBBQRL\ is trained with target quantile $\tau=0.1$. Regret is measured
relative to the finite-horizon expected-return-optimal benchmark. Curves show
the mean over the three held-out seeds, and shaded regions show the descriptive
pointwise variability bands defined in the preceding subsection.}
\label{fig:asset-selling-regret-tau01}
\end{figure}

Figures~\ref{fig:asset-selling-regret-tau05} and
\ref{fig:asset-selling-regret-tau09} address a different question: how well the
learned policies perform with respect to the target quantile itself. Their
vertical axis is therefore the cumulative numerical $\tau$-quantile policy gap
$\widehat{\mathcal G}_{\tau}(k)$, rather than expected-return regret. As
described above, this quantity compares each evaluated policy with the common
approximate-frontier numerical reference and should not be interpreted as the exact
quantile regret $\Reg_\tau(T)$ analyzed theoretically.

For the median target $\tau=0.5$,
Figure~\ref{fig:asset-selling-regret-tau05} shows that \UCBBQRL\ maintains the
smallest cumulative numerical policy gap throughout most of training. At
episode $2000$, its mean cumulative gap is $1.15$, compared with $9.13$ for
UCBVI and $114.43$ for $\varepsilon$-greedy $Q$-learning, which is the
next-smallest value among the remaining benchmarks. Thus, in this
asset-selling instance, explicitly incorporating the buffered quantile
criterion during learning produces a substantially better policy under the
median objective.

\begin{figure}[t]
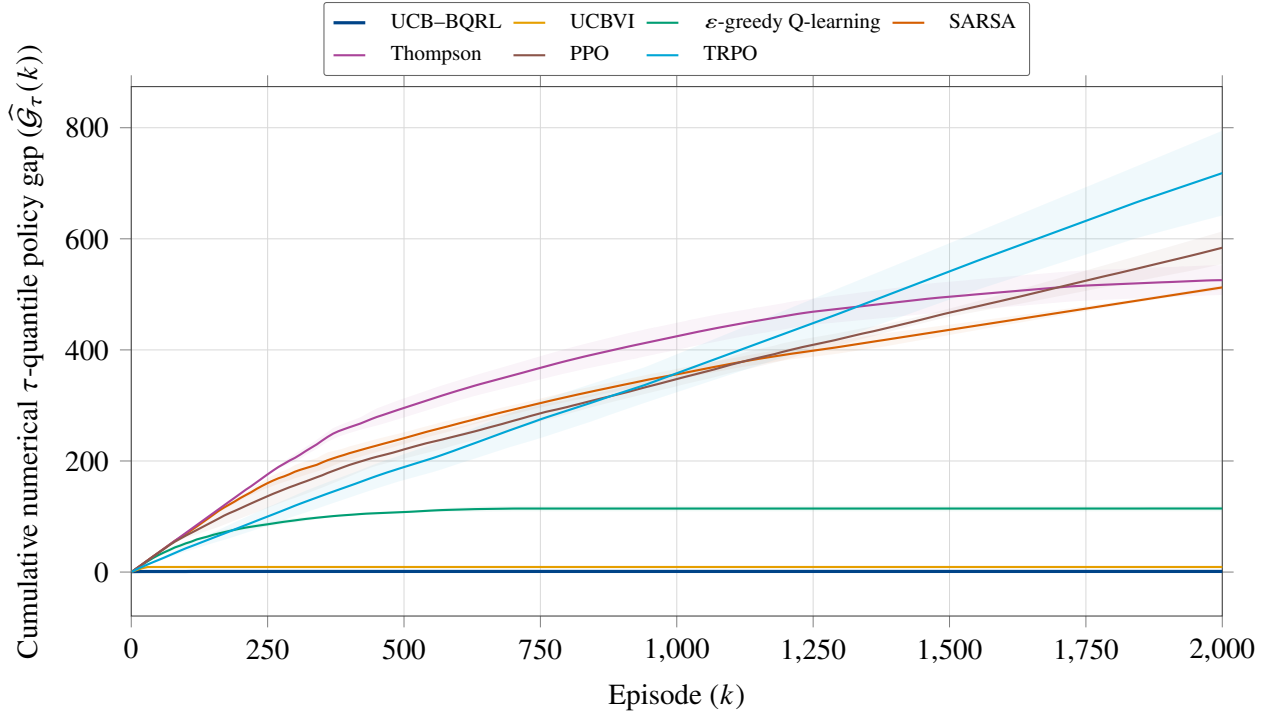

\centering
\assetregretplot
{\QMDPDataDir/AssetSelling_tau0p5_plot.csv}
\caption{Cumulative numerical $\tau$-quantile policy gap
$\widehat{\mathcal G}_{\tau}(k)$ in the asset-selling experiment for
$\tau=0.5$, measured relative to the approximate-frontier numerical reference.
Curves show the mean over the three held-out seeds, and shaded regions show
the descriptive pointwise variability bands defined in the preceding
subsection.}
\label{fig:asset-selling-regret-tau05}
\end{figure}

A similar pattern appears for the upper-tail target $\tau=0.9$ in
Figure~\ref{fig:asset-selling-regret-tau09}. \UCBBQRL\ again attains the
smallest cumulative numerical policy gap. At episode $2000$, its mean value is
$0.79$, compared with $2.79$ for UCBVI and $56.38$ for the next-smallest
baseline. The advantage of \UCBBQRL\ under this criterion is consistent with
its design: unlike the benchmark methods, its planning objective explicitly
depends on the target quantile.

\begin{figure}[t]
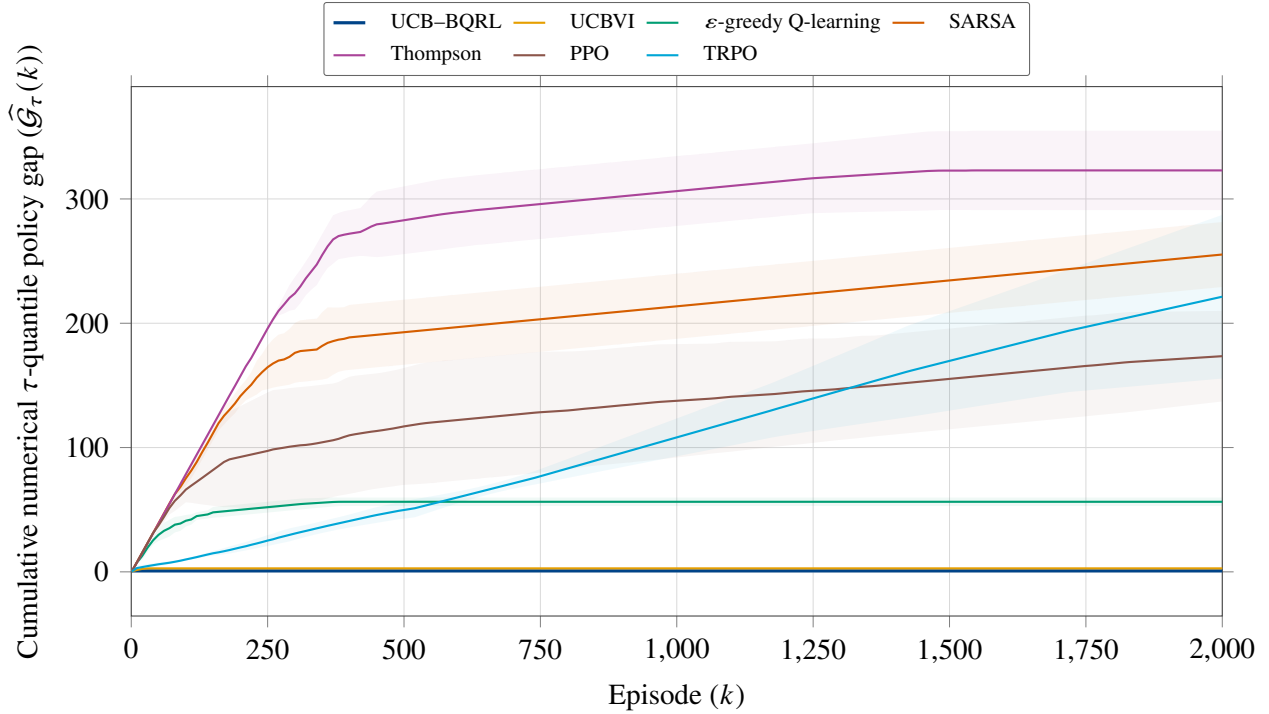

\centering
\assetregretplot
{\QMDPDataDir/AssetSelling_tau0p9_plot.csv}
\caption{Cumulative numerical $\tau$-quantile policy gap
$\widehat{\mathcal G}_{\tau}(k)$ in the asset-selling experiment for
$\tau=0.9$, measured relative to the capped-frontier numerical reference.
Curves show the mean over the three held-out seeds, and shaded regions show
the descriptive pointwise variability bands defined in the preceding
subsection.}
\label{fig:asset-selling-regret-tau09}
\end{figure}

Taken together, the three figures illustrate two different aspects of the
learned policies. Figure~\ref{fig:asset-selling-regret-tau01} shows that
\UCBBQRL\ remains competitive under the conventional expected-return criterion
when trained for a conservative lower-tail objective. In contrast,
Figures~\ref{fig:asset-selling-regret-tau05} and
\ref{fig:asset-selling-regret-tau09} show its advantage when policies are
evaluated according to the quantile-sensitive criterion that motivates the
method. Because the two performance measures have different definitions and
benchmarks, their numerical values should not be compared directly across
figures.

\subsection{Policy and Risk Interpretation}

Figure~\ref{fig:asset-selling-final-policy} shows the final
stage-dependent policy learned by \UCBBQRL\ for each target quantile
$\tau\in\{0.1,0.5,0.9\}$. Final policies are obtained separately for the three
held-out seeds. Because the actions are discrete, averaging these policies
across seeds would generally not produce a valid policy. We therefore display
one selected policy for each value of $\tau$: among the three held-out runs, we
choose the run with the smallest final quantile error, using the
buffered-quantile error and then the seed number as tie-breakers. This selection
rule is used only to illustrate the structure of a learned policy and
deliberately favors the best-performing held-out replicate. Accordingly,
Figure~\ref{fig:asset-selling-final-policy} should not be interpreted as an
average or typical policy across seeds; the cross-seed performance comparisons
are reported separately in the preceding subsection. The absorbing post-sale
state is omitted because no further decision is made after the asset is sold.

The selected policies exhibit a clear threshold structure: for a given
decision period, \UCBBQRL\ continues when the current offer is sufficiently low
and stops once the offer exceeds a quantile-dependent threshold. At the first
decision period, the displayed policy continues for offers $s\leq13$ when
$\tau=0.1$, for $s\leq21$ when $\tau=0.5$, and for $s\leq23$ when
$\tau=0.9$. Equivalently, the policy with $\tau=0.1$ accepts any initial offer
of at least $14$, whereas the policy with $\tau=0.9$ accepts only the maximum
offer, $24$. Thus, the continuation region expands as $\tau$ increases. This
behavior is consistent with the interpretation of the quantile objective:
a smaller value of $\tau$ emphasizes conservative lower-tail performance and
therefore favors accepting a sufficiently good current offer, whereas a larger
value of $\tau$ places greater emphasis on upper-tail outcomes and is therefore
more willing to reject the current offer in pursuit of a better future
realization.

The stopping decision also depends on the remaining horizon. For every value
of $\tau$, the stopping region expands as the end of the episode approaches,
because fewer opportunities remain to obtain a better offer after choosing
$\mathrm{Continue}$. In the final decision period, all three displayed policies
choose $\mathrm{Stop}$ for every offer state, consistent with the dominance
argument in the model formulation. Figure~\ref{fig:asset-selling-final-policy}
therefore illustrates how the learned stopping rule depends jointly on the
target quantile, the current offer, and the remaining number of decision
periods. In particular, changing $\tau$ alters not only the value assigned to a
policy but also the state-wise decisions produced by \UCBBQRL.

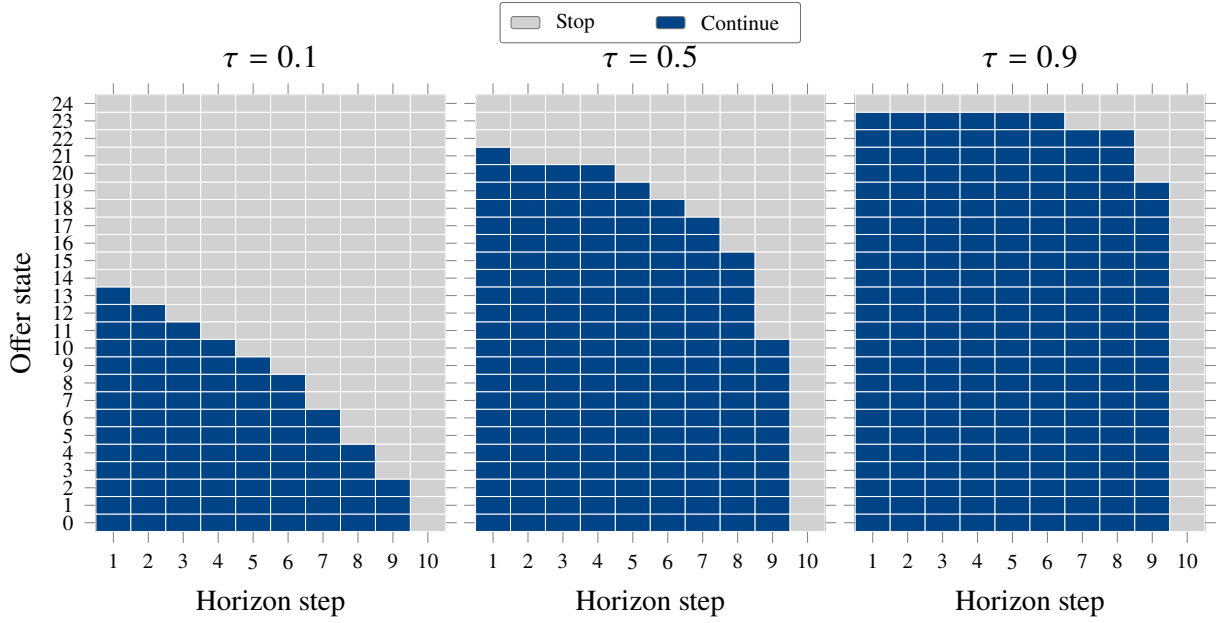
\begin{figure}[t]
\centering

\begin{tikzpicture}

\begin{groupplot}[
    group style={
        group size=3 by 1,
        horizontal sep=0.4cm
    },
    width=0.28\textwidth,
    height=0.35\textwidth,
    scale only axis,
    xmin=0.5,
    xmax=10.5,
    ymin=-0.5,
    ymax=24.5,
    xtick={1,...,10},
    ytick={0,...,24},
    xlabel={Horizon step},
    tick align=outside,
    tick label style={font=\scriptsize},
    label style={font=\small},
    title style={
        font=\normalsize,
        yshift=0mm
    },
    axis line style={black!75}
]

\policyplot[
    ylabel={Offer state}
]{\dataTauLow}{\tau=0.1}

\policyplot[
    yticklabels=\empty
]{\dataTauMedian}{\tau=0.5}

\policyplot[
    yticklabels=\empty
]{\dataTauHigh}{\tau=0.9}

\end{groupplot}

\matrix[
    matrix of nodes,
    nodes={
        font=\scriptsize,
        anchor=center
    },
    draw=black!55,
    fill=white,
    rounded corners=1pt,
    inner xsep=4pt,
    inner ysep=2pt,
    column sep=2pt,
    row sep=0pt,
    anchor=south
]
at ($(group c2r1.north)+(0,0.68cm)$)
{
    |[
        policy legend swatch,
        fill=black!18
    ]| {}
    &
    Stop
    &
    \hspace{7pt}
    &
    |[
        policy legend swatch,
        fill=algNavy
    ]| {}
    &
    Continue
    \\
};

\end{tikzpicture}

\caption{Selected final \UCBBQRL\ policies in the asset-selling experiment
for $\tau\in\{0.1,0.5,0.9\}$. For each target quantile, the displayed policy
is taken from the held-out run with the smallest final quantile error;
buffered-quantile error and seed number are used as tie-breakers. Dark-blue
cells indicate $\mathrm{Continue}$ and light-gray cells indicate
$\mathrm{Stop}$. The displayed policies are selected for structural
illustration and should not be interpreted as averages across seeds. The
absorbing post-sale state is omitted.}
\label{fig:asset-selling-final-policy}
\end{figure}

\section{Discussion}\label{sec:discussion}
We discuss the implications of the proposed buffered quantile learning
framework, focusing on methodological novelty, computational scope, regret
interpretation, and the root-level plateau condition.

\subsection{Toolkit Novelty}
In expectation-based finite-horizon RL, regret analyses exploit the linearity of expectation, which makes transition-estimation errors propagate additively through Bellman recursions. Quantile objectives do not have this structure: quantiles are nonlinear and can change discontinuously under small perturbations of the return distribution. Thus, standard expectation-based optimistic analyses do not directly apply.
Our key idea is to stabilize quantile learning through lower buffering. Instead of planning with the exact \(\tau\)-quantile, \UCBBQRL\space optimizes a lower-buffered criterion that averages quantiles below \(\tau\). This smoothing reduces sensitivity to transition uncertainty while maintaining a connection to the original objective as the buffer shrinks. The resulting regret bound separates the statistical error from learning the transition kernel and the approximation error from using the buffered objective, showing that quantile learning is difficult both because of transition uncertainty and local quantile instability.

\subsection{Computational Aspects}
The optimistic planning step in \UCBBQRL\space is implemented by \EVIBQ. Unlike standard value iteration, it must track full remaining-return distributions rather than scalar expected values, since quantiles depend on the entire return law. Thus, \EVIBQ\space stores finite return laws together with policy-tree labels that encode the corresponding deterministic history-dependent policies. This exact representation defines the optimistic buffered planning problem, but it is best viewed as a theoretical oracle rather than a scalable algorithm. The number of return laws can grow rapidly with the horizon, states, and actions, and optimizing over the confidence set may be computationally costly. Practical versions may use candidate models, sampling, discretization, or pruning, which would introduce planning error; sublinear regret then requires the cumulative planning error to remain sublinear.

\subsection{Comparison with Prior Work}
The regret bound for \UCBBQRL\space preserves the familiar \(\sqrt{T}\) scaling from optimism-based RL. Up to logarithmic factors, the statistical term scales as \(\frac{H^2}{\beta_{T-1}}\sqrt{SAT}\). The factor \(1/\beta_{T-1}\) captures quantile instability: a smaller buffer better approximates the exact \(\tau\)-quantile but increases sensitivity to transition-estimation error. Thus, the buffer balances exact-quantile accuracy and statistical stability. Unlike known-model QMDP planning, the online setting must both estimate the transition kernel and control quantile sensitivity. Confidence sets address transition uncertainty, while lower buffering stabilizes planning. The lower bound shows this challenge is intrinsic: small transition-probability changes can determine whether the exact quantile value is low or high, a phenomenon absent from expectation-based learning.

\subsection{Role of the Root-Level Plateau Parameter}
The root-level plateau parameter \(\rho_\tau\) captures the local structure of
the initial-state return distribution near the target quantile. If, for a fixed
policy, the return distribution has positive mass at its \(\tau\)-quantile
value, then the same value is attained for quantile levels immediately below
\(\tau\). The parameter \(\rho_\tau\) is the smallest such left-neighborhood over
the policies relevant to the regret analysis. This parameter has two roles. First, it determines when the buffered objective
recovers the exact objective: once the buffer is smaller than the plateau width,
the lower-buffered and exact initial-state quantile values coincide, so the
buffering error vanishes after a finite transient under a decreasing buffer
schedule. Second, \(\rho_\tau\) appears in the lower bound: a small plateau width
corresponds to a fragile quantile threshold, where the learner must distinguish
small transition-probability differences to identify whether the target quantile
crosses a critical return level. Thus, quantile RL depends not only on
\(S\), \(A\), \(H\), and \(T\), but also on the local geometry of the return
distribution near the target quantile, with \(\rho_\tau\) playing a role analogous to a margin or gap parameter.
\subsection{Computational Hardness and Future Research}
The hardness result in Section~\ref{sec:hardness} shows that exact point-quantile
evaluation and exact lower-buffered quantile evaluation are PP-hard even for a
fixed policy in a highly restricted finite-horizon MDP. Thus, the computational
difficulty is not only caused by the frontier representation in
Algorithm~\ref{alg:EVI-BQ}; exact quantile evaluation itself is difficult in
general. Consequently, \EVIBQ\space should be viewed as an exact planning
oracle for the regret analysis rather than a scalable generic implementation.

Future research should develop approximate planners with explicit planning-error
guarantees. Possible directions include discretizing confidence sets, sampling
candidate models, pruning nearly duplicate return laws, using distributional
value approximations, and designing deep RL methods for buffered quantile
objectives.

\section*{Code and Data Disclosure}\label{sec:Code and Data Disclosure}The code and data to support the numerical experiments in this paper can be found at \url{https://github.com/MAlipourVaezi/UCB-BQRL}. 

\newpage
\bibliographystyle{plainnat}
\bibliography{references}

\newpage
\section*{Appendix}
\subsection*{Proof of Theorem~\ref{thm:UCB-BQRL-finite}}
Starting from the definition of quantile regret, $\Reg_\tau(T) = \sum_{t=0}^{T-1} \left(V_{\tau,0}^{\pi^\star,P^\star}(\bar s) - V_{\tau,0}^{\pi^t,P^\star}(\bar s)
\right)$. By Lemma~\ref{lem:confidence-set-containment}, with probability at least
\(1-\delta\), both \(P^\star\) and \(P^t\) belong to \(\mathcal C_\delta^t\) for every episode \(t\). Therefore,
Lemma~\ref{lem:buffered-optimistic-decomposition} implies that, with probability
at least \(1-\delta\), for every \(t=0,\ldots,T-1\),
\begin{align}
\Reg_\tau(T)
\le \sum_{t=0}^{T-1}
\Bigg(2\Delta_{\beta_t} + \left(V_{\tau,0}^{\pi^t,P^t,\beta_t}(\bar s) - V_{\tau,0}^{\pi^t,P^\star,\beta_t}(\bar s) \right)\Bigg).
\label{eq:buffered-one-episode-decomp}
\end{align}
Define $L_\delta\coloneqq \log\!\frac{2SATH}{\delta}$, $w_h\coloneqq H-h-1$, and $\Lambda_H\coloneqq \sum_{h=0}^{H-1}w_h$. For each episode \(t\), let \(\mathcal F_t\) denote the sigma-field generated by
all observations before episode \(t\), together with the planned model \(P^t\),
the planned policy \(\pi^t\), and the counts \(N_h^t(s,a)\). Conditional on
\(\mathcal F_t\), the model \(P^t\), the policy \(\pi^t\), and the counts are
fixed.

Fix an episode \(t\). The two root return laws $G_{0,\bar s}^{\pi^t,P^t}$, and $G_{0,\bar s}^{\pi^t,P^\star}$ are supported on \([0,H]\), because each stage reward lies in \([0,1]\). Since
\(\beta_t\le \tau\), Definition~\ref{def:buffered-quantile} gives
\(\ell_{\beta_t}(\tau)=\beta_t\). Applying Lemma~\ref{lem:buffered-quantile-w1} gives
\begin{align}
\left|
V_{\tau,0}^{\pi^t,P^t,\beta_t}(\bar s)
-
V_{\tau,0}^{\pi^t,P^\star,\beta_t}(\bar s)
\right| \le
\frac{1}{\beta_t}
W_1\!\left(
G_{0,\bar s}^{\pi^t,P^t},
G_{0,\bar s}^{\pi^t,P^\star}
\right).
\label{eq:buffered-value-w1-bound}
\end{align}
Next, Lemma~\ref{lem:return-law-w1}, applied with \(P=P^t\) and
\(\bar P=P^\star\), yields
\begin{align}
W_1\!\left(
G_{0,\bar s}^{\pi^t,P^t},
G_{0,\bar s}^{\pi^t,P^\star}
\right)
\le 
\sum_{h=0}^{H-1}
w_h
\mathbb E\!\left[
\mathrm{TV}\!\left(
P_h^t(\cdot\mid S_h^t,A_h^t),
P_h^\star(\cdot\mid S_h^t,A_h^t)
\right)
\,\middle|\,
\mathcal F_t
\right].
\label{eq:w1-tv-trajectory-bound}
\end{align}
On the event from Lemma~\ref{lem:confidence-set-containment}, both
\(P^\star\) and \(P^t\) belong to \(\mathcal C_\delta^t\). Hence
Lemma~\ref{lem:tv-confidence-consequence} implies that, for every \(h,s,a\), $\mathrm{TV}\!\left(
P_h^t(\cdot\mid s,a),
P_h^\star(\cdot\mid s,a)
\right)
\le
f_\delta\!\big(N_h^t(s,a)\big)$. Combining this bound with \eqref{eq:buffered-value-w1-bound} and
\eqref{eq:w1-tv-trajectory-bound}, we obtain
\begin{align}
V_{\tau,0}^{\pi^t,P^t,\beta_t}(\bar s)
-
V_{\tau,0}^{\pi^t,P^\star,\beta_t}(\bar s)
\le
\frac{1}{\beta_t}
\mathbb E\!\left[
\sum_{h=0}^{H-1}
w_h f_\delta\!\big(N_h^t(S_h^t,A_h^t)\big)
\,\middle|\,
\mathcal F_t
\right].
\label{eq:buffered-model-bound}
\end{align}

The left-hand side in \eqref{eq:buffered-model-bound} may be negative, but the
inequality is valid because it follows from an upper bound on the absolute value.

Define $X_t
\coloneqq
\sum_{h=0}^{H-1}
w_h f_\delta\!\big(N_h^t(S_h^t,A_h^t)\big)$, and $\bar X_t \coloneqq \mathbb E[X_t\mid\mathcal F_t]$. Substituting \eqref{eq:buffered-model-bound} into
\eqref{eq:buffered-one-episode-decomp} gives $\Reg_\tau(T) \le 2\sum_{t=0}^{T-1}\Delta_{\beta_t} + \sum_{t=0}^{T-1}\frac{1}{\beta_t}\bar X_t$. Because \(\{\beta_t\}_{t\ge0}\) is non-increasing, $\beta_t\ge \beta_{T-1}$, for each $t=0,\ldots,T-1$. Therefore, $\Reg_\tau(T) \le\; 2\sum_{t=0}^{T-1}\Delta_{\beta_t} + \frac{1}{\beta_{T-1}} \sum_{t=0}^{T-1}\bar X_t$.

It remains to control the predictable sum \(\sum_{t=0}^{T-1}\bar X_t\). Since $f_\delta(n) = c\sqrt{\frac{L_\delta}{\max\{1,n\}}} \le c\sqrt{L_\delta}$ for every \(n\), and since \(\sum_{h=0}^{H-1}w_h=\Lambda_H\), we have $0\le X_t\le c\Lambda_H\sqrt{L_\delta}$, and $0\le \bar X_t\le c\Lambda_H\sqrt{L_\delta}$. Define $Y_t\coloneqq \bar X_t-X_t$. Then \(\{Y_t\}_{t=0}^{T-1}\) is a martingale difference sequence with respect to
the episode filtration, because $\mathbb E[Y_t\mid\mathcal F_t] = \bar X_t-\mathbb E[X_t\mid\mathcal F_t] = 0$. Moreover, $|Y_t|\le c\Lambda_H\sqrt{L_\delta}$. By the one-sided Azuma--Hoeffding inequality, with probability at least \(1-\delta\), $\sum_{t=0}^{T-1}\bar X_t \le \sum_{t=0}^{T-1}X_t + c\Lambda_H\sqrt{2T L_\delta\log\!\frac{1}{\delta}}$. 

By Lemma~\ref{lem:weighted-stage-wise-counting}, $\sum_{t=0}^{T-1}X_t = \sum_{t=0}^{T-1}\sum_{h=0}^{H-1}
w_h f_\delta\!\big(N_h^t(S_h^t,A_h^t)\big) \le 2c\Lambda_H\sqrt{SA T L_\delta}$. Therefore, with probability at least \(1-\delta\) for the martingale event, $\sum_{t=0}^{T-1}\bar X_t
\le
2c\Lambda_H\sqrt{SA T L_\delta}
+
c\Lambda_H\sqrt{2T L_\delta\log\!\frac{1}{\delta}}$. Combining this event with the confidence-containment event from Lemma~\ref{lem:confidence-set-containment}, and applying the union bound, both
events hold simultaneously with probability at least \(1-2\delta\). On this
intersection, $\Reg_\tau(T)
\le\;
2\sum_{t=0}^{T-1}\Delta_{\beta_t}
+
\frac{2c\Lambda_H}{\beta_{T-1}}
\sqrt{SA T\log\!\frac{2SATH}{\delta}}
+
\frac{c\Lambda_H}{\beta_{T-1}}
\sqrt{
2T\log\!\frac{2SATH}{\delta}
\log\!\frac{1}{\delta}
}$. As $\Lambda_H \le \frac{H^2}{2}$, $\Reg_\tau(T)
\le\;
2\sum_{t=0}^{T-1}\Delta_{\beta_t}
+
\frac{c}{\beta_{T-1}}
\sqrt{SATH^4\log\!\frac{2SATH}{\delta}}
+
\frac{c}{\beta_{T-1}}
\sqrt{
\frac{TH^4}{2}\log\!\frac{2SATH}{\delta}
\log\!\frac{1}{\delta}
} \allowbreak$. This proves the theorem.
\myqed
\subsection*{Proof of Theorem~\ref{thm:UCB-BQ-IT_lower}}
We construct a two-state family of MDPs indexed by a distinguished action \(a^\star\in[A]\). The learner does not know \(a^\star\). The only informative randomness in each episode is whether the process moves from the initial state \(s_0\) to a high-reward absorbing state \(s_1\) after the first action. Thus, identifying the optimal quantile policy is statistically equivalent to identifying the best arm in a stochastic Bernoulli bandit.

Let $\mathcal S=\{s_0,s_1\}$, $\bar s=s_0$, and $\mathcal A=[A]$. Define
\(m_\tau=\min\{\tau,1-\tau\}\), and fix a parameter \(\rho\in(0,m_\tau/8]\). Therefore, $\tau-\rho>0$, and $\tau+\rho<1$.

For a fixed distinguished action \(a^\star\in[A]\), define $p(a)\coloneqq P_0^\star(s_1\mid s_0,a)$, for each $a\in[A]$, where $p(a^\star)=1-\tau+\rho$, $p(a)=1-\tau-\rho$ for all $a\neq a^\star$. Equivalently, $P_0^\star(s_0\mid s_0,a^\star)=\tau-\rho$, and $P_0^\star(s_0\mid s_0,a)=\tau+\rho$ for all $a\neq a^\star$. These inequalities ensure that all transition probabilities are valid. The transition kernel is completed as follows.
\begin{itemize}
\item At stage \(h=0\), from \(s_0\), $P_0^\star(s_1\mid s_0,a)=p(a)$, and $P_0^\star(s_0\mid s_0,a)=1-p(a)$, $\forall a\in[A]$.
\item For every \(h\ge1\), both states are absorbing: $P_h^\star(s_0\mid s_0,a)=1$, and P$_h^\star(s_1\mid s_1,a)=1$, $\forall a\in[A]$.
\item For completeness, at \((h,s)=(0,s_1)\), set $P_0^\star(s_1\mid s_1,a)=1$, $\forall a\in[A]$.
\end{itemize}
Thus, the only randomness in an episode is the first transition from \(s_0\);
after time \(1\), the trajectory is deterministic.

The deterministic rewards are $r_0(s_0,a)=0$, $r_h(s_0,a)=0$, and $r_h(s_1,a)=1$, for $h=1,\ldots,H-1$. Therefore, if the process remains in \(s_0\) at time \(1\), the total return is \(0\). If the process moves to \(s_1\) at time \(1\), the total return is $H-1$.

Fix any deterministic history-dependent policy \(\pi\in\Pi_{\mathrm{det}}\), and
let $a\coloneqq \pi_0(s_0)$ be the first action selected at the initial history. Since all later transitions
and rewards are deterministic once \(S_1\) is realized, the total return $G_a\coloneqq \sum_{h=0}^{H-1}r_h(S_h,A_h)$ satisfies $G_a=0, \quad\text{if }S_1=s_0,$ and $G_a=H-1, \quad\text{if }S_1=s_1$. Thus, $\Prob(G_a=0)=1-p(a)$, and $\Prob(G_a=H-1)=p(a)$. For this two-point return law, the left-continuous \(\tau\)-quantile is $Q_\tau(G_a)=0,\quad \text{if }\tau\le 1-p(a)$, and $Q_\tau(G_a)=H-1,\quad \text{if }\tau>1-p(a)$.

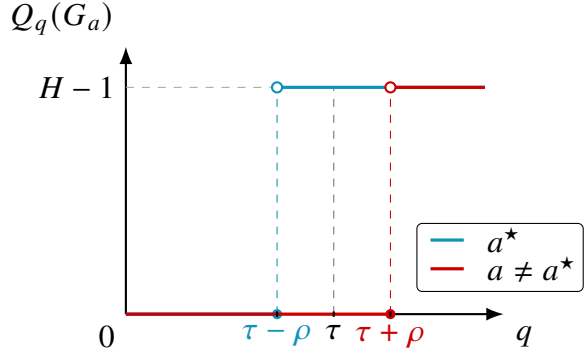
\begin{figure}[H]
\centering
\begin{tikzpicture}[x=5cm,y=3cm,>=Latex]

    % Schematic locations on the q-axis
    \def\qleft{0.40}   % tau - rho
    \def\qtau{0.55}    % tau
    \def\qright{0.70}  % tau + rho
    \def\qmax{1.00}
    \def\ymax{1.18}
    \def\HH{1.00}

    % axes
    \draw[->,thick] (0,0) -- (\qmax,0) node[below right] {$q$};
    \draw[->,thick] (0,0) -- (0,\ymax) node[above left] {$Q_q(G_a)$};

    % origin
    \node[below left] at (0,0) {$0$};

    % high return level
    \draw[dashed,gray!70] (0,\HH) -- (0.83,\HH);
    \node[left] at (0,\HH) {$H-1$};

    % jump guides
    \draw[dashed,cyan!70!black] (\qleft,0) -- (\qleft,\HH);
    \draw[dashed,red!80!black] (\qright,0) -- (\qright,\HH);
    \draw[dashed,black!55] (\qtau,0) -- (\qtau,\HH);

    % quantile function for a^*
    % Left-continuous convention: value at the jump remains the lower value.
    \draw[cyan!70!black,very thick] (0,0) -- (\qleft,0);
    \draw[cyan!70!black,very thick] (\qleft,\HH) -- (0.95,\HH);
    \fill[cyan!70!black] (\qleft,0) circle (1.9pt);
    \filldraw[fill=white,draw=cyan!70!black,thick] (\qleft,\HH) circle (1.9pt);

    % quantile function for a \neq a^*
    \draw[red!80!black,very thick] (0,0) -- (\qright,0);
    \draw[red!80!black,very thick] (\qright,\HH) -- (0.95,\HH);
    \fill[red!80!black] (\qright,0) circle (1.9pt);
    \filldraw[fill=white,draw=red!80!black,thick] (\qright,\HH) circle (1.9pt);

    % x-axis tick labels
    \node[below,cyan!70!black] at (\qleft,0) {$\tau-\rho$};
    \node[below] at (\qtau,0) {$\tau$};
    \node[below,red!80!black] at (\qright,0) {$\tau+\rho$};

    % small ticks on x-axis
    \draw[thick] (\qleft,0.012) -- (\qleft,-0.012);
    \draw[thick] (\qtau,0.012) -- (\qtau,-0.012);
    \draw[thick] (\qright,0.012) -- (\qright,-0.012);

    % legend
    \begin{scope}[shift={(0.77,0.10)}, x=1cm, y=1cm]
        \draw[black, rounded corners=2pt, fill=white] (0,0) rectangle (2.2,0.9);

        \draw[cyan!70!black, very thick] (0.18,0.68) -- (0.62,0.68);
        \node[anchor=west] at (0.78,0.68) {$a^\star$};

        \draw[red!80!black, very thick] (0.18,0.30) -- (0.62,0.30);
        \node[anchor=west] at (0.78,0.30) {$a\neq a^\star$};
    \end{scope}

\end{tikzpicture}
\caption{Quantile functions of the root return law in the information-theoretic
lower-bound construction. For the distinguished action \(a^\star\), the return
is \(0\) with probability \(\tau-\rho\) and \(H-1\) otherwise, so its quantile
function jumps at \(\tau-\rho\). For any suboptimal action \(a\neq a^\star\), the
return is \(0\) with probability \(\tau+\rho\) and \(H-1\) otherwise, so its
quantile function jumps at \(\tau+\rho\). Therefore, at the target level
\(\tau\), the distinguished action has \(Q_\tau(G_{a^\star})=H-1\), whereas any
suboptimal action has \(Q_\tau(G_a)=0\). The open and closed circles reflect the
left-continuous quantile convention.}
\label{fig:it-lb-quantile-jumps}
\end{figure}

For the distinguished action \(a^\star\), $1-p(a^\star)=\tau-\rho<\tau$, so every policy choosing \(a^\star\) at the initial state satisfies $V_{\tau,0}^{\pi,P^\star}(s_0)=H-1$. For every action \(a\neq a^\star\), $1-p(a)=\tau+\rho\ge\tau$, so every policy choosing \(a\neq a^\star\) at the initial state satisfies $V_{\tau,0}^{\pi,P^\star}(s_0)=0$. Hence the unique optimal first action is \(a^\star\), and the per-episode exact \(\tau\)-quantile gap is $H-1$.

We now compute the root-level plateau constant of the constructed instance. For
any policy whose first action is \(a^\star\), the root return law satisfies $\Prob(G_{a^\star}=0)=\tau-\rho$, and $\Prob(G_{a^\star}=H-1)=1-\tau+\rho$. Its \(\tau\)-quantile is \(H-1\), and the left limit of the CDF at \(H-1\) is $F_{a^\star}((H-1)^-)=\tau-\rho$. Therefore, $\tau-F_{a^\star}((H-1)^-)=\rho$. For any policy whose first action is \(a\neq a^\star\), the root return law
satisfies $\Prob(G_a=0)=\tau+\rho$, and $\Prob(G_a=H-1)=1-\tau-\rho$. Its \(\tau\)-quantile is \(0\), and the left limit of the CDF at \(0\) is $F_a(0^-)=0$. Therefore, $\tau-F_a(0^-)=\tau$. Since \(\rho\le m_\tau/8\le \tau/8\), the minimum over all deterministic
policies is $\rho_\tau=\rho$.

Let $I_t\coloneqq \pi_0^t(s_0)$ be the first action selected by the learning algorithm in episode \(t\). Then,
for the MDP indexed by \(a^\star\), $V_{\tau,0}^{\pi^\star,P^\star}(s_0) - V_{\tau,0}^{\pi^t,P^\star}(s_0) = (H-1)\mathbf 1\{I_t\neq a^\star\}$. Thus,
\begin{equation}
\label{eq:it-lower-regret-count}
\Reg_\tau(T)
=
(H-1)
\sum_{t=0}^{T-1}
\mathbf 1\{I_t\neq a^\star\}.
\end{equation}

The interaction induced by this MDP family is equivalent to a stochastic
\(A\)-armed Bernoulli bandit. In episode \(t\), after selecting \(I_t=a\), the
learner observes whether \(S_1^t=s_1\), and $\Prob(S_1^t=s_1\mid I_t=a)=p(a)$. Thus the corresponding Bernoulli bandit has arm means $\mu_{a^\star}=1-\tau+\rho$, $\mu_a=1-\tau-\rho$ $(a\neq a^\star)$, so the gap between the distinguished arm and every suboptimal arm is $\mu_{a^\star}-\mu_a=2\rho$.

By the standard hard-family lower bound for stochastic Bernoulli bandits, there exists a constant \(c_{\mathrm B}>0\), depending at most on the fixed target level \(\tau\), such that for any learning algorithm there exists a choice of \(a^\star\in[A]\) in the one-good-arm Bernoulli family above for which $2\rho\,
\mathbb E\!\left[
\sum_{t=0}^{T-1}
\mathbf 1\{I_t\neq a^\star\}
\right]
\ge
c_{\mathrm B}\sqrt{AT}$. We choose the MDP corresponding to this hard \(a^\star\). Rearranging gives $\mathbb E\!\left[
\sum_{t=0}^{T-1}
\mathbf 1\{I_t\neq a^\star\}
\right]
\ge
\frac{c_{\mathrm B}}{2\rho}\sqrt{AT}$. Since \(\rho_\tau=\rho\), this is $\mathbb E\!\left[
\sum_{t=0}^{T-1}
\mathbf 1\{I_t\neq a^\star\}
\right]
\ge
\frac{c_{\mathrm B}}{2\rho_\tau}\sqrt{AT}$. Substituting this into Equation~\eqref{eq:it-lower-regret-count} yields
$\mathbb E[\Reg_\tau(T)]
=
(H-1)\,
\mathbb E\!\left[
\sum_{t=0}^{T-1}
\mathbf 1\{I_t\neq a^\star\}
\right] \ge
(H-1)\frac{c_{\mathrm B}}{2\rho_\tau}\sqrt{AT}$. Since \(H\ge2\), \(H-1\ge H/2\). Therefore, $\mathbb E[\Reg_\tau(T)]
\ge
\frac{c_{\mathrm B}}{4}
\frac{H}{\rho_\tau}\sqrt{AT}$. Setting \(c_1=c_{\mathrm B}/4\) proves the theorem.
\myqed

\subsection*{Proof of Theorem~\ref{thm:quantile-eval-hardness}}
We reduce from the threshold version of counting \(0/1\)-knapsack solutions.
Given positive integers \(w_1,\ldots,w_n\), a capacity \(K\), and an integer
\(M\), the problem is to decide whether $\#\left\{
I\subseteq[n]:
\sum_{i\in I}w_i\le K
\right\}
\ge M$. This problem is equivalent, up to the convention of ordering subset sums in increasing or decreasing order, to the Kth largest subset problem, which is PP-complete under polynomial-time Turing reductions \cite[Theorem~3]{haase2016complexity}. We therefore use it as the source
problem for the reduction. We restrict attention to nontrivial instances with \(1\le M\le 2^n\) and \(0\le K<\sum_{i=1}^n w_i\). 

Let $W\coloneqq \sum_{i=1}^n w_i$. Construct a finite-horizon MDP with state space $\mathcal S=\{0,1\}$,
one action, and horizon \(H=n+1\). Since there is only one action, there is only one deterministic policy, so evaluation and planning coincide. For each
\(h=0,\ldots,n-1\) and each state \(s\in\{0,1\}\), let $P_h(1\mid s,a)=\frac12$, and $P_h(0\mid s,a)=\frac12$. The final transition is arbitrary and does not affect the total reward. Define deterministic rewards by $r_0(s,a)=0$, and for \(h=1,\ldots,n\), $r_h(1,a)=\frac{w_h}{W}$, $r_h(0,a)=0$. All rewards lie in \([0,1]\), and all parameters in the constructed instance
have polynomial encoding length.

Let \(X_h\coloneqq \mathbf 1\{S_h=1\}\) for \(h=1,\ldots,n\). Then
\(X_1,\ldots,X_n\) are independent Bernoulli\((1/2)\) random variables, and the
total return under the unique policy is $R = \sum_{h=1}^{n}\frac{w_h}{W}X_h$. Therefore, $\mathbb P\!\left(R\le \frac{K}{W}\right)
=
2^{-n}
\#\left\{
I\subseteq[n]:
\sum_{i\in I}w_i\le K
\right\}$. Set $c\coloneqq \frac{K}{W}$, and $\tau\coloneqq \frac{M-\frac12}{2^n} = \frac{2M-1}{2^{n+1}}$. Then \(\tau\in(0,1)\).

First, consider the point-quantile objective. If the knapsack count is at least
\(M\), then $\mathbb P(R\le c)\ge \frac{M}{2^n}>\tau$, and hence \(Q_\tau(R)\le c\). Conversely, if the knapsack count is less than \(M\), then the count is at most \(M-1\), so $\mathbb P(R\le c)\le\frac{M-1}{2^n}<\tau$. Thus \(Q_\tau(R)>c\). Therefore, deciding whether \(Q_\tau(R)\le c\) decides the threshold counting-knapsack problem. Exact QMDP evaluation is therefore PP-hard.

Now consider the lower-buffered objective. Set $\beta\coloneqq 2^{-(n+1)}$. Then \(\beta\in(0,\tau]\) and $\tau-\beta = \frac{M-1}{2^n}$. If the knapsack count is at least \(M\), then $\mathbb P(R\le c)\ge\frac{M}{2^n}>\tau$, so \(Q_u(R)\le c\) for every \(u\le\tau\). Hence $Q_\tau^\beta(R)\le c$. Conversely, if the knapsack count is less than \(M\), then $\mathbb P(R\le c)\le\frac{M-1}{2^n}=\tau-\beta$. Therefore, for every \(u\in(\tau-\beta,\tau]\), the CDF of \(R\) has not reached
level \(u\) at \(c\), and hence \(Q_u(R)>c\). Since the support of \(R\) lies on the lattice \(\{0,1/W,\ldots,1\}\), this implies $Q_u(R)\ge c+\frac1W$ for every $u\in(\tau-\beta,\tau]$. The endpoint \(u=\tau-\beta\) has Lebesgue measure zero, so it does not affect
the lower-buffered integral. Consequently, $Q_\tau^\beta(R)>c$. Thus deciding whether \(Q_\tau^\beta(R)\le c\) also decides the threshold
counting-knapsack problem. Exact lower-buffered quantile evaluation is therefore PP-hard.

Finally, the corresponding planning problems contain evaluation as a special
case: in the constructed MDP, there is only one action and hence only one policy.
Therefore, exact point-quantile QMDP planning and exact lower-buffered quantile
planning are PP-hard.
\myqed
\subsection*{Proof of Proposition~\ref{prop:evibq}}
Fix \(P\in\mathcal C_\delta^t\). We first show by backward induction that, for
every \(h=0,\ldots,H\) and \(s\in\mathcal S\), the frontier
\(\mathfrak D_h^P(s)\) contains exactly the finite return laws achievable from
state \(s\) at stage \(h\) under deterministic continuation policies and model
\(P\), together with labels for policies that generate those laws.

At \(h=H\), no reward remains. Hence the only achievable return law is the point
mass \(\delta_0\), which is exactly how the algorithm initializes
\(\mathfrak D_H^P(s)\).

Now suppose the statement holds at stage \(h+1\). Fix a state \(s\) and a
deterministic continuation policy starting at \((h,s)\). Such a policy first
chooses an action \(a\in\mathcal A\). After the next state \(S_{h+1}=s_i\) is
observed, the policy follows a deterministic continuation policy from
\((h+1,s_i)\). By the induction hypothesis, the return law of this continuation
is some \(D_i\in\mathfrak D_{h+1}^P(s_i)\). Therefore, the total return law from
\((h,s)\) is $\operatorname{Merge}
\left(
\bigcup_{i=1}^{S}
\left\{
\left(
P_h(s_i\mid s,a)\rho,\,
r_h(s,a)+x
\right)
:
(\rho,x)\in D_i
\right\}
\right)$, which Algorithm~\ref{alg:EVI-BQ} adds to \(\mathfrak D_h^P(s)\).

Conversely, every law added to \(\mathfrak D_h^P(s)\) is generated by a
deterministic continuation policy: select the stored action \(a\) at state \(s\)
and stage \(h\), and after observing \(S_{h+1}=s_i\), follow the continuation
policy encoded by the stored label \(L_i\). Hence the frontier
\(\mathfrak D_h^P(s)\) contains exactly the achievable finite return laws. This
closes the induction.

Applying the result at \(h=0\) and \(s=\bar s\), the algorithm maximizes
\(Q_\tau^{\beta_t}(D)\) over exactly all root return laws achievable under
deterministic policies for the fixed model \(P\). Therefore, $\widehat V_{\tau,0}^{P,\beta_t}(\bar s)
=
\max_{\pi\in\Pi_{\mathrm{det}}}
V_{\tau,0}^{\pi,P,\beta_t}(\bar s)$. The final maximization over \(P\in\mathcal C_\delta^t\) gives $V_{\tau,0}^{\pi^t,P^t,\beta_t}(\bar s)
=
\max_{P\in\mathcal C_\delta^t}
\max_{\pi\in\Pi_{\mathrm{det}}}
V_{\tau,0}^{\pi,P,\beta_t}(\bar s)$. Since the maximization is over \(P\in\mathcal C_\delta^t\), the returned model
satisfies \(P^t\in\mathcal C_\delta^t\). This proves the proposition.
\myqed
\subsection*{Proof of Corollary~\ref{cor:buffered-UCB-BQrl-log-rate}}
By Theorem~\ref{thm:UCB-BQRL-finite}, $\Reg_\tau(T)
\le\;
2\sum_{t=0}^{T-1}\Delta_{\beta_t}
+
\frac{2c\Lambda_H}{\beta_{T-1}}
\sqrt{SA T\log\!\frac{2SATH}{\delta}}+
\frac{c\Lambda_H}{\beta_{T-1}}
\sqrt{
2T\log\!\frac{2SATH}{\delta}
\log\!\frac{1}{\delta}
}$. For the logarithmic schedule, $\beta_{T-1}
=
\frac{\tau}{\log(\mathrm e+T-1)}$. Therefore, $\frac{1}{\beta_{T-1}}
=
\frac{\log(\mathrm e+T-1)}{\tau}$. Moreover, by Lemma~\ref{lem:constant-buffering-cost}, $2\sum_{t=0}^{T-1}\Delta_{\beta_t}
\le
2H K_\tau$. Substituting these two bounds into Theorem~\ref{thm:UCB-BQRL-finite} completes the proof.
\myqed
\subsection*{Technical Lemmas}
\begin{lemma}[High-probability confidence-set containment]
\label{lem:confidence-set-containment}
Let \(f_\delta\) be defined as in Equation~\eqref{eq:confset}. Suppose that, in
each episode \(t\), the planning step (Line~\ref{alg:UCB--BQRL-L10} in Algorithm~\ref{alg:UCB--BQRL} and equivalently, Algorithm~\ref{alg:EVI-BQ}) returns a model
\(P^t\in\mathcal C_\delta^t\). Then, with probability at least \(1-\delta\),
both the true model and the returned optimistic model belong to the confidence
set $P^\star\in\mathcal C_\delta^t$ and $P^t\in\mathcal C_\delta^t$, for all $t=0,\ldots,T-1$.
\end{lemma}

\begin{proof}{Proof of Lemma~\ref{lem:confidence-set-containment}}
Define the event $\mathcal E_\delta
\coloneqq
\big\{
\big\|P_h^\star(\cdot\mid s,a)-\widehat P_h^t(\cdot\mid s,a)\big\|_1
\le
f_\delta\!\big(N_h^t(s,a)\big),
\ \forall t,h,s,a
\big\} \allowbreak$. We first show that \(\mathbb P(\mathcal E_\delta)\ge 1-\delta\). Let $L_\delta\coloneqq \log\!\frac{2SATH}{\delta}$. Fix a tuple \((t,h,s,a)\), and condition on \(N_h^t(s,a)=n\). If \(n=0\), then
by the choice of \(c\), $f_\delta(0)=c\sqrt{L_\delta}\ge 2$. Because the \(\ell_1\)-distance between two probability vectors is at most \(2\), the confidence inequality holds automatically when \(n=0\).

Now suppose \(n\ge1\). Conditional on \(N_h^t(s,a)=n\), the observed next states
from visits to \((h,s,a)\) are i.i.d. draws from
\(P_h^\star(\cdot\mid s,a)\). By Lemma~\ref{lem:weissman}, $\mathbb P\!\big(
\big\|
P_h^\star(\cdot\mid s,a)-\widehat P_h^t(\cdot\mid s,a)
\big\|_1
>
f_\delta(n)\,|\,
N_h^t(s,a)=n
\big) \le (2^S-2)\exp\!\big(-\frac{n f_\delta(n)^2}{2}\big)$. For \(n\ge1\), the definition of \(f_\delta\) gives $n f_\delta(n)^2
=
c^2 L_\delta$. Therefore, $(2^S-2)\exp\!\big(-\frac{n f_\delta(n)^2}{2}\big)=(2^S-2)\exp\!\big(-\frac{c^2L_\delta}{2}\big)$. The stated lower bound on \(c\) ensures $(2^S-2)\exp\!\left(-\frac{c^2L_\delta}{2}\right)
\le
\frac{\delta}{SATH}$. Taking a union bound over all \(t=0,\ldots,T-1\),
\(h=0,\ldots,H-1\), \(s\in\mathcal S\), and \(a\in\mathcal A\), we obtain $\mathbb P(\mathcal E_\delta^c)\le \delta$. Thus, $\mathbb P(\mathcal E_\delta)\ge 1-\delta$.

On \(\mathcal E_\delta\), the true kernel satisfies, for every \(t,h,s,a\), $\left\|
P_h^\star(\cdot\mid s,a)-\widehat P_h^t(\cdot\mid s,a)
\right\|_1
\le
f_\delta\!\big(N_h^t(s,a)\big)$. By the definition of \(\mathcal C_\delta^t\), this means $P^\star\in\mathcal C_\delta^t$, for all $t=0,\ldots,T-1$. The returned model satisfies $P^t\in\mathcal C_\delta^t$, for all $t=0,\ldots,T-1$, by the stated construction of the planning step. Hence, on \(\mathcal E_\delta\),
both \(P^\star\) and \(P^t\) belong to \(\mathcal C_\delta^t\) for every episode
\(t\). Since \(\mathbb P(\mathcal E_\delta)\ge1-\delta\), the result follows.
\end{proof}

\begin{lemma}[Weissman’s $\ell_1$ concentration]\label{lem:weissman}
Let $X_1,\dots,X_n$ be i.i.d.\ on $[S]\coloneqq\{1,\dots,S\}$ with $\Prob(X_1=i)=p_i$. Define the empirical distribution
$\widehat p_i \coloneqq \tfrac1n\sum_{t=1}^n \mathbf 1\{X_t=i\}$. Then $\Prob\!\bigl(\|\widehat p-p\|_1\ge \eps\bigr)\ \le\ (2^S-2)\,\exp\!\left(-\frac{n\eps^2}{2}\right)$.
\end{lemma}
\begin{proof}{Proof of Lemma \ref{lem:weissman}}
For any $x\in\mathbb R^S$, $\|x\|_1 \;=\; \max_{v\in\{-1,+1\}^S} v^\top x$. If, in addition, $\sum_i x_i=0$, then the maximizers cannot be $v=\mathbf 1$ or $v=-\mathbf 1$ (since $v^\top x=0$ for those two), where $\mathbf 1$ is an all-one vector. Hence
\begin{equation}\label{eq:l1-hypercube}
\|x\|_1 \;=\; \max_{v\in\mathcal V} v^\top x,
\qquad
\mathcal V \;\coloneqq\; \{-1,+1\}^S\setminus\{\mathbf 1,-\mathbf 1\},
\end{equation}
and consequently $\bigl\{x:\|x\|_1 \ge \eps\bigr\}\ \subseteq\ \bigcup_{v\in\mathcal V}\ \bigl\{x:v^\top x \ge \eps\bigr\}$.

Fix $v\in\mathcal V$ and define $Y_t^{(v)}\coloneqq v_{X_t}\in\{-1,+1\}$. Then $v^\top \widehat p \;=\; \sum_{i=1}^S v_i \widehat p_i
\;=\; \frac1n \sum_{t=1}^n v_{X_t}
\;=\; \frac1n \sum_{t=1}^n Y_t^{(v)}$, and $\mathbb E\,Y_t^{(v)} \;=\; \sum_{i=1}^S p_i v_i \;=\; v^\top p$.
Hence $v^\top(\widehat p - p) \;=\; \frac1n \sum_{t=1}^n \bigl(Y_t^{(v)} - \mathbb E Y_t^{(v)}\bigr)$, a mean of i.i.d.\ centered random variables taking values in $[-1,1]$. By Hoeffding’s inequality,
\begin{align}\label{eq:hoeffding}
\Prob\!\left(v^\top(\widehat p - p) \ge \eps\right) \le \exp\!\Bigl(-\frac{2 n \eps^2}{(1-(-1))^2}\Bigr) = \exp\!\left(-\frac{n\eps^2}{2}\right).
\end{align}

Combining Equations \eqref{eq:l1-hypercube}, the union bound, and Equation \eqref{eq:hoeffding}, $\Prob\!\bigl(\|\widehat p - p\|_1 \ge \eps\bigr)
\;\le\; \sum_{v\in\mathcal V} \Prob\!\bigl(v^\top(\widehat p - p) \ge \eps\bigr)
\;\le\; |\mathcal V| \, \exp\!\left(-\frac{n\eps^2}{2}\right)$. Since $|\mathcal V|=2^S-2$, the stated bound follows.
\end{proof}

\begin{lemma}[Buffered optimistic decomposition]
\label{lem:buffered-optimistic-decomposition}
Fix an episode \(t\). Suppose \(P^\star\in\mathcal C_\delta^t\), and suppose that
\((P^t,\pi^t)\) is returned by the exact buffered optimistic oracle in
Algorithm~\ref{alg:EVI-BQ}. Then $V_{\tau,0}^{\pi^\star,P^\star}(\bar s)
-
V_{\tau,0}^{\pi^t,P^\star}(\bar s)
\le
2\Delta_{\beta_t}
+
\Big(
V_{\tau,0}^{\pi^t,P^t,\beta_t}(\bar s)
-
V_{\tau,0}^{\pi^t,P^\star,\beta_t}(\bar s)
\Big)$.

\end{lemma}
\begin{proof}{Proof of Lemma~\ref{lem:buffered-optimistic-decomposition}}
Fix episode \(t\). Add and subtract
\(V_{\tau,0}^{\pi^\star,P^\star,\beta_t}(\bar s)\),
\(V_{\tau,0}^{\pi^t,P^t,\beta_t}(\bar s)\), and
\(V_{\tau,0}^{\pi^t,P^\star,\beta_t}(\bar s)\). This gives 
$V_{\tau,0}^{\pi^\star,P^\star}(\bar s)
-
V_{\tau,0}^{\pi^t,P^\star}(\bar s)
=
\left[
V_{\tau,0}^{\pi^\star,P^\star}(\bar s)
-
V_{\tau,0}^{\pi^\star,P^\star,\beta_t}(\bar s)
\right] +
\left[
V_{\tau,0}^{\pi^\star,P^\star,\beta_t}(\bar s)
-
V_{\tau,0}^{\pi^t,P^t,\beta_t}(\bar s)
\right]
+
\left[
V_{\tau,0}^{\pi^t,P^t,\beta_t}(\bar s)
-
V_{\tau,0}^{\pi^t,P^\star,\beta_t}(\bar s)
\right]
+
\left[
V_{\tau,0}^{\pi^t,P^\star,\beta_t}(\bar s)
-
V_{\tau,0}^{\pi^t,P^\star}(\bar s)
\right]$. By the definition of \(\Delta_{\beta_t}\), the first and fourth bracketed terms
are each bounded above by \(\Delta_{\beta_t}\). Thus, $V_{\tau,0}^{\pi^\star,P^\star}(\bar s)
-
V_{\tau,0}^{\pi^t,P^\star}(\bar s)
\le
2\Delta_{\beta_t}
+
\left[
V_{\tau,0}^{\pi^\star,P^\star,\beta_t}(\bar s)
-
V_{\tau,0}^{\pi^t,P^t,\beta_t}(\bar s)
\right]
+
\left[
V_{\tau,0}^{\pi^t,P^t,\beta_t}(\bar s)
-
V_{\tau,0}^{\pi^t,P^\star,\beta_t}(\bar s)
\right]$. It remains to control the middle bracket. Since
\(P^\star\in\mathcal C_\delta^t\), the pair \((P^\star,\pi^\star)\) is feasible for the exact buffered optimistic planning problem. Therefore, $V_{\tau,0}^{\pi^\star,P^\star,\beta_t}(\bar s)
\le
\max_{P\in\mathcal C_\delta^t}
\max_{\pi\in\Pi_{\mathrm{det}}}
V_{\tau,0}^{\pi,P,\beta_t}(\bar s)$. By Proposition~\ref{prop:evibq}, Algorithm~\ref{alg:EVI-BQ} returns a pair \((P^t,\pi^t)\) satisfying $V_{\tau,0}^{\pi^t,P^t,\beta_t}(\bar s)
=\max_{P\in\mathcal C_\delta^t}
\max_{\pi\in\Pi_{\mathrm{det}}}
V_{\tau,0}^{\pi,P,\beta_t}(\bar s)$. Hence $V_{\tau,0}^{\pi^\star,P^\star,\beta_t}(\bar s)
-V_{\tau,0}^{\pi^t,P^t,\beta_t}(\bar s)
\le 0$. Substituting this into the previous display proves the lemma.
\end{proof}

\begin{lemma}[Buffered quantile sensitivity]
\label{lem:buffered-quantile-w1}
Let \(\mu\) and \(\nu\) be probability measures supported on \([0,H]\). For
\(u\in(0,1]\), $\left|
Q_q^\beta(\mu)-Q_q^\beta(\nu)
\right|
\le
\frac{1}{\ell_\beta(q)}W_1(\mu,\nu)$.
\end{lemma}

\begin{proof}{Proof of Lemma~\ref{lem:buffered-quantile-w1}}
Because \(\mu\) and \(\nu\) are supported on the bounded interval \([0,H]\), their
quantile functions are bounded and therefore integrable on \((0,1)\). Lemma~\ref{lem:w1-quantile-representation} gives $W_1(\mu,\nu)
=
\int_0^1
\left|
Q_u(\mu)-Q_u(\nu)
\right|\,du$. This identity holds for probability measures on the real line with finite first
moments, and the finite first-moment condition is automatic here because the
support is contained in \([0,H]\). By the definition of the lower-buffered quantile (Definition~\ref{def:buffered-quantile}), $Q_q^\beta(\mu)-Q_q^\beta(\nu)
 =
\frac{1}{\ell_\beta(q)}
\int_{q-\ell_\beta(q)}^{q}
Q_u(\mu)\,du
- \frac{1}{\ell_\beta(q)}
\int_{q-\ell_\beta(q)}^{q}
Q_u(\nu)\,du
 = \frac{1}{\ell_\beta(q)}
\int_{q-\ell_\beta(q)}^{q}
\Big(Q_u(\mu)-Q_u(\nu)\Big)\,du$.
Taking absolute values and applying the triangle inequality for integrals yields $\left|
Q_q^\beta(\mu)-Q_q^\beta(\nu)
\right|
\le
\frac{1}{\ell_\beta(q)}
\int_{q-\ell_\beta(q)}^{q}
\left|
Q_u(\mu)-Q_u(\nu)
\right|\,du$. The interval \([q-\ell_\beta(q),q]\) is a subinterval of \([0,1]\). Therefore $\int_{q-\ell_\beta(q)}^{q}
\left|
Q_u(\mu)-Q_u(\nu)
\right|\,du
\le
\int_0^1
\left|
Q_u(\mu)-Q_u(\nu)
\right|\,du
=
W_1(\mu,\nu)$. Substituting this bound into the preceding display proves $\left|
Q_q^\beta(\mu)-Q_q^\beta(\nu)
\right|
\le
\frac{1}{\ell_\beta(q)}W_1(\mu,\nu)$. 
\end{proof}

\begin{lemma}[Quantile representation of \(W_1\) on the real line]
\label{lem:w1-quantile-representation}
Let \(\mu\) and \(\nu\) be probability measures on \(\mathbb R\) with finite first
moments. Let $F_\mu(x)\coloneqq \mu((-\infty,x])$, and $F_\nu(x)\coloneqq \nu((-\infty,x])$ be their CDFs, and define their left-continuous generalized inverses by $F_\mu^{-1}(u)\coloneqq \inf\{x\in\mathbb R:F_\mu(x)\ge u\}$, $F_\nu^{-1}(u)\coloneqq \inf\{x\in\mathbb R:F_\nu(x)\ge u\}$, for all $u\in(0,1)$. Equivalently, \(F_\mu^{-1}(u)=Q_u(\mu)\) and
\(F_\nu^{-1}(u)=Q_u(\nu)\). Then $W_1(\mu,\nu)
=
\int_0^1
\left|
F_\mu^{-1}(u)-F_\nu^{-1}(u)
\right|\,du
=
\int_0^1
\left|
Q_u(\mu)-Q_u(\nu)
\right|\,du$.
\end{lemma}

\begin{proof}{Proof of Lemma~\ref{lem:w1-quantile-representation}}
We use the standard primal definition of the \(1\)-Wasserstein distance on the
real line. For probability measures \(\mu\) and \(\nu\) with finite first moments, $W_1(\mu,\nu)
\coloneqq
\inf_{\gamma\in\Pi(\mu,\nu)}
\int_{\mathbb R\times\mathbb R}|x-y|\,d\gamma(x,y)$,
where \(\Pi(\mu,\nu)\) is the set of all couplings of \(\mu\) and \(\nu\), that is,
all probability measures on \(\mathbb R\times\mathbb R\) whose first marginal is
\(\mu\) and whose second marginal is \(\nu\). The quantile representation below
is a standard one-dimensional characterization of \(W_1\); see, for example,
\cite[Chapter 2]{villani2009ot}. For completeness, we give the argument.

Let \(U\sim\mathrm{Unif}(0,1)\), and define $X\coloneqq F_\mu^{-1}(U)$, and $Y\coloneqq F_\nu^{-1}(U)$. By the generalized inverse transform, \(X\sim\mu\) and \(Y\sim\nu\). Therefore, the joint law of \((X,Y)\) is a coupling of \(\mu\) and \(\nu\). Hence, by the
definition of \(W_1\), $W_1(\mu,\nu)
\le
\mathbb E|X-Y|
=
\int_0^1
\left|
F_\mu^{-1}(u)-F_\nu^{-1}(u)
\right|\,du$. Conversely, in one dimension, the monotone coupling generated by the common
uniform random variable \(U\) is optimal for the absolute-value transport cost
\(|x-y|\). Therefore,
$W_1(\mu,\nu)
=
\mathbb E\left|
F_\mu^{-1}(U)-F_\nu^{-1}(U)
\right|
=
\int_0^1
\left|
F_\mu^{-1}(u)-F_\nu^{-1}(u)
\right|\,du$. Finally, since the left-continuous quantile operator satisfies $Q_u(\mu)=F_\mu^{-1}(u)$, and $Q_u(\nu)=F_\nu^{-1}(u)$, we obtain $W_1(\mu,\nu)
=
\int_0^1
\left|
Q_u(\mu)-Q_u(\nu)
\right|\,du$. This proves the claim.
\end{proof}

\begin{lemma}[Return-law perturbation for a fixed induced policy]
\label{lem:return-law-w1}
Fix an induced policy \(\pi=\pi[\mu,\Gamma]\), initialized at \(q_0=\tau\), and
fix two transition kernels \(P\) and \(\bar P\). Let
\(G_{0,\bar s}^{\pi,P}\) and \(G_{0,\bar s}^{\pi,\bar P}\) be the
corresponding root return laws. Then $W_1\!\left(
G_{0,\bar s}^{\pi,P},
G_{0,\bar s}^{\pi,\bar P}
\right)
\le
\sum_{h=0}^{H-1}(H-h-1)\,
\mathbb E_{\pi,\bar P}
\left[
\mathrm{TV}\!\left(
P_h(\cdot\mid \bar S_h, \bar A_h),
\bar P_h(\cdot\mid \bar S_h,\bar A_h)
\right)
\right]$, where, for two probability vectors \(p,p'\in\Delta^S\), $\mathrm{TV}(p,p')\coloneqq \frac12\|p-p'\|_1$.

\end{lemma}

\begin{proof}{Proof of Lemma~\ref{lem:return-law-w1}}
We couple two trajectories. The trajectory
\((\bar S_h,\bar A_h,\bar q_h)_{h=0}^{H-1}\) evolves under \(\bar P\), and the
trajectory \((S_h,A_h,q_h)_{h=0}^{H-1}\) evolves under \(P\). Both start from $S_0=\bar S_0=\bar s$, and $q_0=\bar q_0=\tau$. As long as the two physical histories are identical, the two internal quantile levels are also identical, because both are obtained by applying the same
deterministic update maps \(\Gamma\) to the same observed history. Therefore, as long as no mismatch has occurred, $S_h=\bar S_h$, $q_h=\bar q_h$,and $A_h=\bar A_h$. At stage \(h\), conditional on no previous mismatch and on the common state--action pair \((\bar S_h,\bar A_h)\), define $p_i\coloneqq P_h(s_i\mid \bar S_h,\bar A_h)$, $
\bar p_i\coloneqq \bar P_h(s_i\mid \bar S_h,\bar A_h)$,
for all $i=1,\ldots,S$. We couple \(S_{h+1}\sim p\) and \(\bar S_{h+1}\sim \bar p\) by a maximal coupling. To construct this coupling explicitly, define the common mass $m_i\coloneqq \min\{p_i,\bar p_i\}$, and $m\coloneqq \sum_{i=1}^{S}m_i$. The quantity \(m_i\) is the amount of probability mass that both distributions
assign to the same next state \(s_i\). Hence, if the two coupled variables are
to be equal and take value \(s_i\), the probability of this event cannot exceed
either marginal probability \(p_i\) or \(\bar p_i\). Therefore, for any coupling, $\mathbb P(S_{h+1}=\bar S_{h+1}=s_i) \le \min\{p_i,\bar p_i\} = m_i$. Summing over \(i=1,\ldots,S\), every coupling must satisfy $\mathbb P(S_{h+1}=\bar S_{h+1})=\sum_{i=1}^{S}\mathbb P(S_{h+1}=\bar S_{h+1}=s_i)
\le
\sum_{i=1}^{S}m_i
=
m$. Thus, no coupling can make the two next states equal with probability larger
than \(m\).

The maximal coupling attains this upper bound. With probability \(m\), draw a common next state \(s_i\) with probability \(m_i/m\), and set $S_{h+1}=\bar S_{h+1}=s_i$. With the remaining probability \(1-m\), draw \(S_{h+1}\) and \(\bar S_{h+1}\) from the residual distributions $\frac{p_i-m_i}{1-m}$, and $\frac{\bar p_i-m_i}{1-m}$, respectively. These residual distributions have disjoint supports: for each state \(s_i\), at most one of \(p_i-m_i\) and \(\bar p_i-m_i\) is positive. Consequently, on the residual event, the two next states are different almost surely.

This construction has the correct marginals. Indeed, for each \(i\), $\mathbb P(S_{h+1}=s_i)=m_i+(p_i-m_i)=p_i$, and $\mathbb P(\bar S_{h+1}=s_i)=m_i+(\bar p_i-m_i)=\bar p_i$. Moreover, it attains the largest possible equality probability: $\mathbb P\!\left(S_{h+1}=\bar S_{h+1}\,\middle|\,\text{no previous mismatch},\bar S_h,\bar A_h\right)=m=\sum_{i=1}^{S}\min\{p_i,\bar p_i\}$. Therefore, $\mathbb P\!\left(S_{h+1}\neq \bar S_{h+1}\,\middle|\,\text{no previous mismatch},\bar S_h,\bar A_h\right) =1-\sum_{i=1}^{S}\min\{p_i,\bar p_i\}$. Finally, using the identity $|p_i-\bar p_i|=p_i+\bar p_i-2\min\{p_i,\bar p_i\}$, and summing over \(i=1,\ldots,S\), we obtain $\frac12\sum_{i=1}^{S}|p_i-\bar p_i|=\frac12\left(\sum_{i=1}^{S}p_i+\sum_{i=1}^{S}\bar p_i-2\sum_{i=1}^{S}\min\{p_i,\bar p_i\}\right)$. Since \(p\) and \(\bar p\) are probability vectors, this becomes $\frac12\sum_{i=1}^{S}|p_i-\bar p_i|=1-\sum_{i=1}^{S}\min\{p_i,\bar p_i\}$. Hence, $1-\sum_{i=1}^{S}\min\{p_i,\bar p_i\}=\mathrm{TV}(p,\bar p)$, we obtain $\mathbb P\!\left(S_{h+1}\neq \bar S_{h+1}\,\middle|\,\text{no previous mismatch},\bar S_h,\bar A_h\right)=\mathrm{TV}(p,\bar p)$. Equivalently, $\mathbb P\!\left(S_{h+1}\neq \bar S_{h+1}\,\middle|\,\text{no previous mismatch},\bar S_h,\bar A_h \right)=\mathrm{TV}\!\left(P_h(\cdot\mid \bar S_h,\bar A_h),\bar P_h(\cdot\mid \bar S_h,\bar A_h)\right)$.

Let \(\mathcal M_h\) be the event that the first mismatch occurs in the transition from stage \(h\) to stage \(h+1\). These events are mutually exclusive, because at most one transition can be the first transition at which the two coupled trajectories separate. If no mismatch ever occurs, then the two trajectories are identical at every stage, and the two cumulative returns are equal. Now suppose that \(\mathcal M_h\) occurs. By the definition of the first mismatch, the two trajectories are identical up to and including stage \(h\). In particular, $S_k=\bar S_k$,\qquad $q_k=\bar q_k$,and $A_k=\bar A_k$, for all $k=0,\ldots,h$. Since the reward is deterministic and depends only on the state--action pair, we have $r_k(S_k,A_k)=r_k(\bar S_k,\bar A_k)$, for all $k=0,\ldots,h$. Thus, the cumulative rewards can differ only from stages \(h+1\) through \(H-1\). Therefore, on \(\mathcal M_h\), $\left|\sum_{k=0}^{H-1} r_k(S_k,A_k) - \sum_{k=0}^{H-1} r_k(\bar S_k,\bar A_k)\right|=\left|\sum_{k=h+1}^{H-1}\Big(r_k(S_k,A_k)-r_k(\bar S_k,\bar A_k)\Big)\right| \le \sum_{k=h+1}^{H-1}\left|r_k(S_k,A_k)-r_k(\bar S_k,\bar A_k) \right|$. Because each reward lies in \([0,1]\), each absolute difference in the last sum is at most \(1\). There are exactly \(H-h-1\) terms in the sum \(k=h+1,\ldots,H-1\). Hence, on \(\mathcal M_h\), $\left|\sum_{k=0}^{H-1} r_k(S_k,A_k)-\sum_{k=0}^{H-1} r_k(\bar S_k,\bar A_k)\right|\le H-h-1$. Combining the cases over the mutually exclusive events \(\mathcal M_0,\ldots,\mathcal M_{H-1}\), and noting that the left-hand side is zero if no mismatch occurs, we obtain the path-wise bound $\left|\sum_{k=0}^{H-1} r_k(S_k,A_k)-\sum_{k=0}^{H-1} r_k(\bar S_k,\bar A_k)\right|\le \sum_{h=0}^{H-1}(H-h-1)\mathbf 1\{\mathcal M_h\}$. Taking expectations gives $W_1\!\left(G_{0,\bar s}^{\pi,P},G_{0,\bar s}^{\pi,\bar P} \right) \le \sum_{h=0}^{H-1}(H-h-1)\mathbb P(\mathcal M_h)$.

It remains to upper-bound \(\mathbb P(\mathcal M_h)\). By the maximal coupling construction, $\mathbb P(\mathcal M_h) \le \mathbb E_{\pi,\bar P} \left[\mathrm{TV}\!\left(P_h(\cdot\mid \bar S_h, \bar A_h), \bar P_h(\cdot\mid \bar S_h,\bar A_h) \right) \right]$. The expectation on the right is under the \(\bar P\)-trajectory because, before the first mismatch, the common trajectory has the same distribution as the trajectory generated by \(\pi\) under \(\bar P\). Substituting this bound into the preceding display proves the lemma.
\end{proof}

\begin{lemma}[Total-variation consequence of the confidence set]
\label{lem:tv-confidence-consequence}
Fix an episode \(t\). Suppose that \(P^\star\in\mathcal C_\delta^t\) and \(P^t\in\mathcal C_\delta^t\). Then, for every \(h\in\{0,\ldots,H-1\}\) and \((s,a)\in\mathcal S\times\mathcal A\), $\mathrm{TV}\!\left(P_h^t(\cdot\mid s,a), P_h^\star(\cdot\mid s,a) \right) \le f_\delta\!\big(N_h^t(s,a)\big)$.
\end{lemma}

\begin{proof}{Proof of Lemma~\ref{lem:tv-confidence-consequence}}
Fix \(h,s,a\). Since \(P^\star\in\mathcal C_\delta^t\), the definition of the confidence set gives $\left\|P_h^\star(\cdot\mid s,a)-\widehat P_h^t(\cdot\mid s,a) \right\|_1 \le f_\delta\!\big(N_h^t(s,a)\big)$. Similarly, since \(P^t\in\mathcal C_\delta^t\), we also have $\left\|
P_h^t(\cdot\mid s,a) - \widehat P_h^t(\cdot\mid s,a) \right\|_1 \le f_\delta\!\big(N_h^t(s,a)\big)$. Applying the triangle inequality in \(\ell_1\) gives $\left\|P_h^t(\cdot\mid s,a)-P_h^\star(\cdot\mid s,a)\right\|_1\le \left\|P_h^t(\cdot\mid s,a)-\widehat P_h^t(\cdot\mid s,a)\right\|_1  +\left\|\widehat P_h^t(\cdot\mid s,a)-P_h^\star(\cdot\mid s,a)\right\|_1 \le 2f_\delta\!\big(N_h^t(s,a)\big)$. By the definition of total variation for finite probability vectors, $\mathrm{TV}\!\big(P_h^t(\cdot\mid s,a),P_h^\star(\cdot\mid s,a)\big)=\frac12\big\|P_h^t(\cdot\mid s,a) - P_h^\star(\cdot\mid s,a) \big\|_1$. Substituting the previous \(\ell_1\)-bound yields $\mathrm{TV}\!\left(P_h^t(\cdot\mid s,a), P_h^\star(\cdot\mid s,a) \right) \le f_\delta\!\big(N_h^t(s,a)\big)$. Because \(h,s,a\) were arbitrary, the result holds for all triples.
\end{proof}

\begin{lemma}[Weighted stage-wise counting bound]
\label{lem:weighted-stage-wise-counting}
Let \(w_h\coloneqq H-h-1\), \(\Lambda_H\coloneqq \sum_{h=0}^{H-1}w_h=\frac{H(H-1)}{2}. \) For any realized sequence generated by the algorithm, $\sum_{t=0}^{T-1}\sum_{h=0}^{H-1} \frac{w_h}{\sqrt{\max\{1,N_h^t(S_h^t,A_h^t)\}}} \le 2\Lambda_H\sqrt{SA T}$. Consequently, with \(f_\delta(n)\) defined in Equation~\eqref{eq:confset}, $\sum_{t=0}^{T-1}\sum_{h=0}^{H-1} w_h f_\delta\!\big(N_h^t(S_h^t,A_h^t)\big) \le 2c\Lambda_H \sqrt{SA T\log\!\frac{2SATH}{\delta}}$.
\end{lemma}

\begin{proof}{Proof of Lemma~\ref{lem:weighted-stage-wise-counting}}
Fix a stage \(h\). For a state--action pair \((s,a)\), let $n_{h,s,a} \coloneqq \sum_{t=0}^{T-1} \mathbf 1\{S_h^t=s,\ A_h^t=a\}$ be the total number of visits to \((s,a)\) at stage \(h\) over the \(T\) episodes. On the first visit to \((h,s,a)\), the pre-episode count is \(0\). On the second visit, the pre-episode count is \(1\), and so on. Therefore, the contribution of this fixed triple \((h,s,a)\) to the unweighted inverse-square-root sum is at most $\sum_{j=0}^{n_{h,s,a}-1}\frac{1}{\sqrt{\max\{1,j\}}}$. If \(n_{h,s,a}=0\), this contribution is zero. If \(n_{h,s,a}\ge1\), then $\sum_{j=0}^{n_{h,s,a}-1}\frac{1}{\sqrt{\max\{1,j\}}} = 1+\sum_{j=1}^{n_{h,s,a}-1}\frac{1}{\sqrt j} \le 2\sqrt{n_{h,s,a}}$. Thus, for the fixed stage \(h\), $\sum_{t=0}^{T-1} \frac{1}{\sqrt{\max\{1,N_h^t(S_h^t,A_h^t)\}}} \le 2\sum_{s,a}\sqrt{n_{h,s,a}}$. By Cauchy--Schwarz, $\sum_{s,a}\sqrt{n_{h,s,a}} \le \sqrt{SA\sum_{s,a}n_{h,s,a} }$. At each fixed stage \(h\), exactly one state--action pair is visited in each episode, so $\sum_{s,a}n_{h,s,a}=T$. Therefore, $\sum_{t=0}^{T-1}\frac{1}{\sqrt{\max\{1,N_h^t(S_h^t,A_h^t)\}}} \le 2\sqrt{SA T}$. Multiplying by \(w_h\) and summing over \(h\) gives $\sum_{t=0}^{T-1}\sum_{h=0}^{H-1} \frac{w_h}{\sqrt{\max\{1,N_h^t(S_h^t,A_h^t)\}}} \le 2\sqrt{SA T}\sum_{h=0}^{H-1}w_h = 2\Lambda_H\sqrt{SA T}$. Finally, multiplying by \(c\sqrt{\log(2SATH/\delta)}\) gives the second claim.
\end{proof}

\begin{lemma}[Constant buffering cost under the logarithmic buffer]
\label{lem:constant-buffering-cost}
Let \(\rho_\tau\) be the root-level left-plateau threshold in Definition~\ref{def:root-plateau-threshold}. for each $t$ choose \(\beta_t = \frac{\tau} {\log(\mathrm e+t)},\) and define \(K_\tau \coloneqq \left\lceil \exp\!\left(\frac{\tau}{\rho_\tau}\right)-\mathrm e \right\rceil.\) Then $\Delta_{\beta_t}=0 \qquad \forall\,t\ge K_\tau$. Moreover, for every \(T\ge1\), $\sum_{t=0}^{T-1}\Delta_{\beta_t} \le \sum_{t=0}^{K_\tau-1}\Delta_{\beta_t} \le H K_\tau$.
\end{lemma}

\begin{proof}{Proof of Lemma~\ref{lem:constant-buffering-cost}}
By Definition~\ref{def:root-plateau-threshold}, for every \(\pi\in\Pi_{\mathrm{det}}\), $\rho_\tau \le \tau - F_{\tau,0,\bar s}^{\pi,P^\star}\!\left(\left(V_{\tau,0}^{\pi,P^\star}(\bar s)\right)^- \right)$. Moreover, \(\rho_\tau>0\) because \(V_{\tau,0}^{\pi,P^\star}(\bar s)\) is the left-continuous \(\tau\)-quantile of \(G_{0,\bar s}^{\pi,P^\star}\), and the minimum is taken over the finite class \(\Pi_{\mathrm{det}}\). By Lemma~\ref{lem:buffer-gap}, $\Delta_\beta=0$, $\forall\,0<\beta\le\rho_\tau$. Therefore, under the logarithmic buffer schedule, \(\Delta_{\beta_t}=0\) whenever $\frac{\tau}{\log(\mathrm e+t)} \le \rho_\tau$. Since \(\log(\mathrm e+t)>0\), this condition is equivalent to $\log(\mathrm e+t)\ge \frac{\tau}{\rho_\tau}$. Exponentiating both sides gives $\mathrm e+t \ge \exp\!\left(\frac{\tau}{\rho_\tau}\right)$, or equivalently, $t \ge \exp\!\left(\frac{\tau}{\rho_\tau}\right)-\mathrm e$. Thus the first integer episode from which the condition \(\beta_t\le\rho_\tau\) is guaranteed to hold is $K_\tau = \left\lceil \exp\!\left(\frac{\tau}{\rho_\tau}\right)-\mathrm e \right\rceil$. Hence $\Delta_{\beta_t}=0$, $\forall\,t\ge K_\tau$. It follows that, for every \(T\ge1\), $\sum_{t=0}^{T-1}\Delta_{\beta_t} = \sum_{t=0}^{\min\{T,K_\tau\}-1}\Delta_{\beta_t} \le \sum_{t=0}^{K_\tau-1}\Delta_{\beta_t}$. Finally, since all cumulative rewards lie in \([0,H]\), both \(V_{\tau,0}^{\pi,P^\star}(\bar s)\) and \(V_{\tau,0}^{\pi,P^\star,\beta_t}(\bar s)\) lie in \([0,H]\). Therefore, $0\le \Delta_{\beta_t}\le H$. Consequently, $\sum_{t=0}^{K_\tau-1}\Delta_{\beta_t} \le H K_\tau$. Multiplying by \(2\) gives $2\sum_{t=0}^{T-1}\Delta_{\beta_t} \le 2H K_\tau$. This proves the lemma.
\end{proof}

\begin{lemma}[Vanishing root-level buffering gap]
\label{lem:buffer-gap}
Suppose \(\mathcal S\), \(\mathcal A\), and \(H\) are finite and rewards are deterministic and bounded in \([0,1]\). Then there exists \(\rho_\tau>0\) such that $\Delta_\beta=0$, $\forall\,0<\beta\le \rho_\tau$.
\end{lemma}

\begin{proof}{Proof of Lemma~\ref{lem:buffer-gap}}
Fix \(\pi\in\Pi_{\mathrm{det}}\), and let \(F_\pi\) be the CDF of \(G_{0,\bar s}^{\pi,P^\star}\). Set $x_\pi\coloneqq Q_\tau\!\left(G_{0,\bar s}^{\pi,P^\star}\right)$, and $\rho_\pi\coloneqq \tau-F_\pi(x_\pi^-)$. Since \(x_\pi\) is the left-continuous \(\tau\)-quantile, $F_\pi(x_\pi^-)<\tau\le F_\pi(x_\pi)$, and hence \(\rho_\pi>0\). Because \(\Pi_{\mathrm{det}}\) is finite, the quantity $\rho_\tau\coloneqq \min_{\pi\in\Pi_{\mathrm{det}}}\rho_\pi$ is strictly positive. Now fix \(0<\beta\le\rho_\tau\). For every \(\pi\in\Pi_{\mathrm{det}}\) and every \(u\in[\tau-\beta,\tau]\), we have $F_\pi(x_\pi^-)<u\le F_\pi(x_\pi)$, and therefore $Q_u\!\left(G_{0,\bar s}^{\pi,P^\star}\right)
=
Q_\tau\!\left(G_{0,\bar s}^{\pi,P^\star}\right)$. It follows that $V_{\tau,0}^{\pi,P^\star,\beta}(\bar s)
=
\frac{1}{\beta}
\int_{\tau-\beta}^{\tau}
Q_u\!\left(G_{0,\bar s}^{\pi,P^\star}\right)\,du
=
V_{\tau,0}^{\pi,P^\star}(\bar s)$ for every \(\pi\in\Pi_{\mathrm{det}}\). Hence \(\Delta_\beta=0\) for every \(0<\beta\le\rho_\tau\).
This proves the claim.
\end{proof}

\end{document}